\documentclass{article}

\usepackage[preprint,nonatbib]{neurips_2019}

\usepackage[utf8]{inputenc} 
\usepackage[T1]{fontenc}    
\usepackage[hidelinks]{hyperref}       
\usepackage{url}            
\usepackage{booktabs}       
\usepackage{amsfonts}       
\usepackage{nicefrac}       
\usepackage{microtype}      

\usepackage{amssymb,amsmath,amsthm,dsfont,bm,mathtools}
\usepackage{color,graphicx,epsfig}
\usepackage{enumitem}
\usepackage{multirow,bigstrut,rotating}
\usepackage{accents}
\usepackage{enumitem}
\setlist[itemize]{leftmargin=0.6cm}
\usepackage{listliketab}
\usepackage{caption,subcaption}
\usepackage[title]{appendix}

\newcommand{\tblsmall}{\fontsize{6.5pt}{9pt}\selectfont}
\newcommand{\tbltiny}[1]{\fontsize{3.5pt}{7pt}\selectfont{#1}}

\usepackage{dblfloatfix}
\usepackage{tabularx}
\usepackage{rotating}
\usepackage{adjustbox}
\makeatletter
\newcommand{\settitle}{\@maketitle}
\makeatother

\newcommand{\allsets}{{\mathcal{S}}} 
\newcommand{\smplset}{{T}}
\newcommand{\dist}{{D}}

\renewcommand{\u}{{\bm{u}}}
\newcommand{\feat}{{\varphi}} 
\newcommand{\f}{{F}} 
\renewcommand{\r}{{r}} 
\newcommand{\w}{{w}} 

\newcommand{\agg}{{g}}
\newcommand{\Agg}{{\cal{G}}}

\newcommand{\nonlin}{{\mu}}

\newcommand{\Viol}{{V}}

\newcommand{\Nonviol}{{U}} 

\newcommand{\err}[2]{{\varepsilon^{{#1}}_{{#2}}}}

\newcommand{\sep}{{\Delta}}
\newcommand{\iso}{{\Lambda}}
\newcommand{\coap}{{\Gamma}}
\newcommand{\vcap}{{\kappa}}

\newcommand{\loss}{{\Delta}}

\newcommand{\succarg}[1]{{\succ^{}_{\!{#1}}}}
\renewcommand{\sep}{{\Omega}}
\newcommand{\Fplus}{{\F_+}}
\newcommand{\Fplussup}[1]{{\F_+^{{#1}}}}
\newcommand{\fbar}{{\bar{f}}}
\newcommand{\Base}{{\F}}
\newcommand{\nn}{{\mathcal{N}}}
\newcommand{\sig}{{\sigma}}
\newcommand{\mech}{{M}}
\newcommand{\btheta}{{\bm{\theta}}}

\newcommand{\SDW}{{SDW}} 
\newcommand{\SDE}{{SDE}} 
\newcommand{\SDA}{{SDA}}
\newcommand{\ubar}[1]{\text{\b{$#1$}}}

\newcommand{\N}{\mathbb{N}}                     

\makeatletter
\newcommand{\fsc}[1]{%
	\begingroup\escapechar=\m@ne
	\xdef\fsc@font{\expandafter\string\the\textfont\mathgroup}%
	\endgroup
	\text{%
		\edef\fsc@size{\f@size}%
		\expandafter\split@name\fsc@font\@nil
		\usefont{\f@encoding}{\f@family}{\f@series}{\f@shape}%
		\fontsize{.8\dimexpr\fsc@size}{\z@}\selectfont
		#1%
	}%
}
\makeatother

\newcommand{\red}[1]{{}} 
\newcommand{\blue}[1]{{#1}} 

\newcommand\todo[1]{{}} 

\newcommand{\nir}[1]{{}} 
\newcommand{\kojin}[1]{{}} 

\newcommand{\beq}{\begin{equation}}
\newcommand{\eeq}{\end{equation}}
\newcommand{\bal}{\begin{align}}
\newcommand{\eal}{\end{align}}
\newcommand\expect[2]{\mathbb{E}_{#1}{\left[ {#2} \right]}}

\DeclareMathOperator*{\argmax}{argmax}
\DeclareMathOperator*{\argmin}{argmin}
\DeclareMathOperator*{\avg}{avg}

\newcommand\inner[1]{\langle {#1} \rangle}

\newcommand{\1}[1]{\mathds{1}{\{{#1}\}}}
\newcommand{\one}[1]{\mathds{1}_{\{{#1}\}}}
\newcommand{\naive}{{na\"{\i}ve}}

\newcommand{\X}{{\cal{X}}}

\newcommand{\F}{{\cal{F}}}

\newcommand{\G}{{\cal{G}}}

\newcommand{\R}{{\mathbb{R}}}
\newcommand{\yhat}{{\hat{y}}}

\newtheorem{lemma}{Lemma}
\newtheorem{theorem}{Theorem}
\newtheorem{definition}{Definition}
\newtheorem{corollary}{Corollary}
\newtheorem{claim}{Claim}
\newtheorem{proposition}{Proposition}

\title{Predicting Choice with Set-Dependent Aggregation}

\author{%
  Nir Rosenfeld\\
  Harvard University\\
   \And
   Kojin Oshiba \\
   Harvard University \\
   \And
   Yaron Singer \\
   Harvard University \\
}

\begin{document}

\maketitle


\begin{abstract}
Providing users with alternatives to choose from 
is an essential component in many online platforms,
making the accurate prediction of choice vital to their success.
A renewed interest in learning choice models
has led to significant progress in modeling power,
but most current methods are either 
limited in the types of choice behavior they capture,
cannot be applied to large-scale data,
or both.

Here we propose a learning framework for predicting choice
that is accurate, versatile, theoretically grounded, and scales well.
Our key modeling point is that to account for how humans choose,
predictive models must capture certain set-related invariances.
Building on recent results in economics,
we derive a class of models that can express \emph{any} behavioral choice pattern,
enjoy favorable sample complexity guarantees,
and can be efficiently trained end-to-end.
Experiments on three large choice datasets
demonstrate the utility of our approach.

\end{abstract}


\section{Introduction} \label{sec:intro}


One of the most prevalent activities of online users is \emph{choosing}.
In almost any online platform, users constantly face choices:
what to purchase, who to follow, where to dine, what to watch, 
and even simply where to click.
As the prominence of online services becomes ever more reliant on such choices,
the accurate prediction of choice is quickly becoming vital to their success.
The  availability of large-scale choice data
has spurred hopes of feasible individual-level prediction,
and many recent works have been devoted to
the modeling and prediction of choice
\cite{benson2016relevance, ragain2016pairwise, kleinberg2017comparison, 
	kleinberg2017human, mottini2017deep, shah2017simple,
	negahban2018learning, chierichetti2018learning, overgoor2018choosing}.

In a typical choice scenario, a user is presented with a set of items
$s=\{x^{(1)},\dots,x^{(k)}\}$, $x^{(i)} \in \X$, called the \emph{choice set}.
Of these, the user chooses an item $y \in s$. 
In economics	 this is known as the problem of \emph{discrete choice} \cite{luce1959individual}.
We let the collection of choice sets $\allsets$
include all sets of at most $n$ items, $\allsets  = \cup_{k \le n} \X^k$.
We follow the standard machine learning setup 
and assume choice sets and choices are drawn i.i.d. from an unknown
joint distribution $\dist$.
Given a set of $m$ examples $\smplset=\{ (s_i,y_i) \}_{i=1}^m$ sampled from $\dist$,
our goal is to learn a choice predictor $h(s)$ that generalizes well to unseen sets,
i.e., has low expected error w.r.t. $\dist$.

A natural way to predict choice is to learn an item score function $f(x)$
from a class $\F$,
used to model the predicted probability of choosing $x$ from $s$ as
$P_s(x) = e^{f(x)} / \sum\nolimits_{x' \in s} e^{f(x')}$.
If the learned $f \in \F$ scores chosen items higher than their alternatives,
then $P_s(x)$ will lead to useful predictions.
This approach may seem appealing,
but is in fact constrained by an undesired artifact known as the
\emph{Independence of Irrelevant Alternatives} (IIA):
\begin{definition} \label{def:iia}  \cite{luce1959individual} 
	$P$ is said to satisfy \textbf{IIA} if
	for all $s \in \allsets$ and for any $a,b \in s$,
	it holds that
	\[
	P_{\{a,b\}}(a) / P_{\{a,b\}}(b) = P_s(a) / P_s(b)
	\]
\end{definition}
IIA states that the likelihood of choosing $a$ over $b$
should not depend on what other alternatives are available.
While item probabilities can depend on $s$,
predictions based on the argmax rule $h_f(s) = \argmax_{x \in s} P_s(x)$
are clearly \emph{independent} of $s$.
IIA is therefore a rigid constraint imposing a fundamental
limitation on what choice behavior can be expressed,
and cannot be mitigated by simply increasing the complexity of
functions in $\F$ (e.g., adding layers or clever non-linearities).

From a practical point of view, this is discouraging,
as there is ample empirical evidence that real choice data exhibits
regular and consistent violations of IIA
(see \cite{rieskamp2006extending} for an extensive survey).
This has led to a surge of interest in machine learning models that go beyond IIA
\cite{oh2014learning, osogami2014restricted, benson2016relevance, ragain2016pairwise,
	otsuka2016deep, ragain2018choosing, chierichetti2018learning, seshadri2019discovering, 
	pfannschmidt2019learning}.


The \naive\ way to avoid IIA is to directly model all possible subsets,
but this is likely to make learning intractable
\cite{mcfadden1977application, seshadri2019discovering}.
A common approach for resolving this difficulty is to impose structure,
typically in the form of a probabilistic choice model encoding certain inter-item dependencies.
While allowing for violations of IIA, this approach presents several practical limitations.
First, explicitly modeling the dependency structure
restricts \emph{how} IIA can be violated,
which may not necessarily align with the choice behavior in the data.
Second,
many of these models are designed to satisfy choice axioms
or asymptotic properties (e.g., consistency),
rather than to optimize predictive performance from finite-sample data.
Finally,
surprisingly few of these methods apply to large-scale online choice data,
where the number of instances can be prohibitively large,
choice sets rarely appear more than once,
and items can be complex structured objects whose
number is virtually unbounded.


To complement most current works,
and motivated by the growing need choice models that are accurate and scalable,
here we propose a framework for choice prediction
that is \emph{purely discriminative}, and hence directly optimizes for accuracy.
Our framework is based on the idea of \emph{set-dependent aggregation},
a principled approach in economics for modeling set-dependent choice
\cite{ambrus2015rationalising},
that to the best of our knowledge has not yet been considered from a machine learning perspective.

The key challenge is in designing a choice model that balances expressivity and efficiency:
On the one hand, the model should be flexible enough to capture
the manner in which set-dependence is expressed in the data,
regardless of its complexity and in a data-driven manner.
On the other hand, it must be structured enough to allow for learning to be statistically efficient and computationally tractable.
In this paper, we show that set-dependent aggregation achieves both.
Our framework makes the following contributions:
\begin{itemize}
	\item 
	\textbf{Efficient and scalable discriminative training.}
	Our approach is geared towards predictive performance:
	it directly optimizes for accuracy,
	can be efficiently trained end-to-end,
	and scales well to realistically large and complex choice-prediction scenarios.
	
	\item
	\textbf{Rigorous error bounds.}
	We bound both approximation and estimation error of aggregation.
	Our results show how the complexity of the hypothesis class
	controls the \emph{type} and \emph{number} of accountable violations,
	uncovering a data-driven mechanism for targeting important violations.
	
	\item
	\textbf{Behavioral inductive bias.}
	Aggregation can provably express \emph{any} form of violation \cite{ambrus2015rationalising},
	but may require an intractable number of parameters to do so.
	To control for overfitting,	we infuse our model with inductive bias
	taken from behavioral decision theory,
	thus balancing flexibility and specificity. 
	
	\item
	\textbf{Thorough empirical evaluation.}
	We conduct experiments on three large choice datasets---flight itineraries, hotel reservations, and news recommendations---demonstrating the utility of our approach.
	Our analysis reveals the means by which aggregation improves performance,
	providing empirical support to our theoretical results.
\end{itemize}
Overall, our work presents a practical and theoretically-grounded approach
for predicting choice.\\

\textbf{Paper organization.}
We begin with a review of the related literature (Sec. \ref{sec:related}).
We then present our model of set-dependent aggregation in Sec. \ref{sec:model}.
This is followed in Sec. \ref{sec:theory} by our main theoretical results,
bounding the approximation error (Sec. \ref{sec:apx_err})
and estimation error (Sec. \ref{sec:est_err}) of aggregators.
Finally, we present our experimental results in Sec. \ref{sec:experiments},
and give concluding remarks in Sec. \ref{sec:conclusions}.



\subsection{Related material} \label{sec:related}

IIA begins with Luce's Axiom of Choice \cite{luce1959individual},
which for a certain noise distribution, induces
the popular Multinomial Logit model (MNL) \cite{mcfadden1973conditional,train2009discrete}.
Two common extensions--- Nested MNL and Mixed MNL ---relax IIA
by grouping items via a tree structure \cite{mcfadden1978modeling}
or modeling a mixture population \cite{mcfadden2000mixed}, respectively.
These, however, pose restrictive assumptions on the nature of violations,
require elaborate hand-coded auxiliary inputs,
and are in many cases intractable.
Although recent progress has alleviated some difficulties
\cite{oh2014learning, benson2016relevance, chierichetti2018learning},
applying these models in practice remains difficult.
Recent works have proposed other probabilistic models that deviate from IIA,
by modeling pairwise utilities \cite{ragain2018choosing},
$k^{\text{th}}$-order interactions \cite{ragain2016pairwise,seshadri2019discovering},
and general subset relations \cite{benson2018discrete}.
These, however, do not optimize for predictive accuracy,
and rarely apply to complex choice settings with many items and sparse choice sets.
Others suggest discriminative solutions,
but these include models that tend to be either over-specific or over-general
\cite{mottini2017deep,pfannschmidt2019learning}
and provide no guarantees as to what violations can be captured. 


Our work draws on the literature of utility aggregation (or ``multi-self'' models),
studied extensively in economics
\cite{kalai2002rationalizing, fudenberg2006dual, manzini2007sequentially, green2009choice,
	ambrus2015rationalising},
psychology
\cite{tversky1972elimination, shafir1993reason, tversky1993context},
and marketing
\cite{kivetz2004alternative}.
While these typically focus on mathematical tractability or behavioral plausibility,
our focus is on statistical and computational aspects relevant to machine learning.
Our work is largely inspired by recent results in economics on the
expressivity of aggregation \cite{ambrus2015rationalising}.
While they give worst-case guarantees for unparameterized items
under the realizability assumption,
our results study the statistical behavior of
parametric aggregation.
To the best of our knowledge, our paper is the first to provide rigorous
finite-sample error bounds for set-dependent aggregation in the context of choice prediction.

\section{Proposed Model} \label{sec:model} 

\begin{figure}[!t]
\centering
\begin{subfigure}[b]{0.32\linewidth}
	\centering\includegraphics[width=1\linewidth]{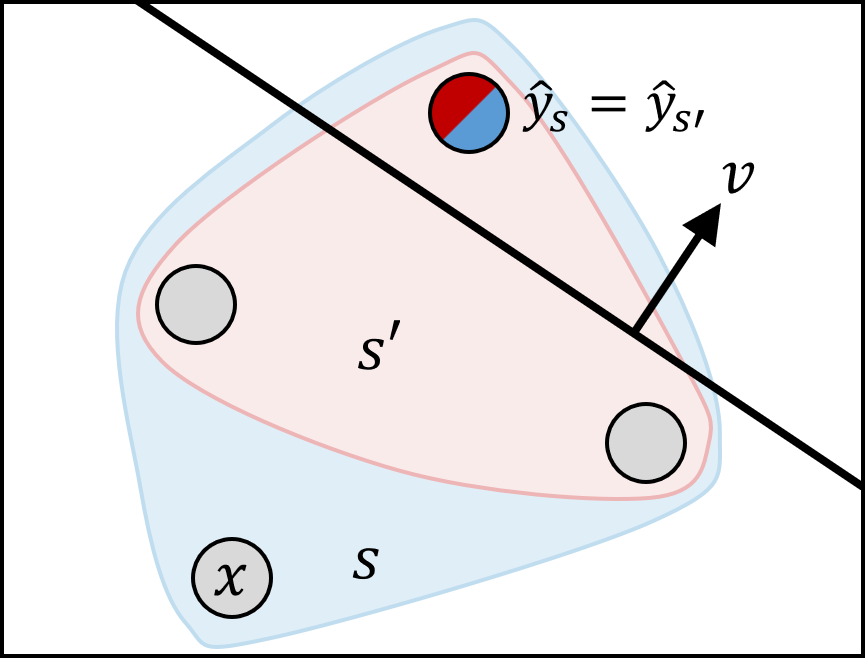}
	\caption{\textbf{Linear}: $\inner{v,\f(x)}$
		\label{fig:viol_viz_lin}}
\end{subfigure}
\,\,
\begin{subfigure}[b]{0.32\linewidth}
	\centering\includegraphics[width=1\linewidth]{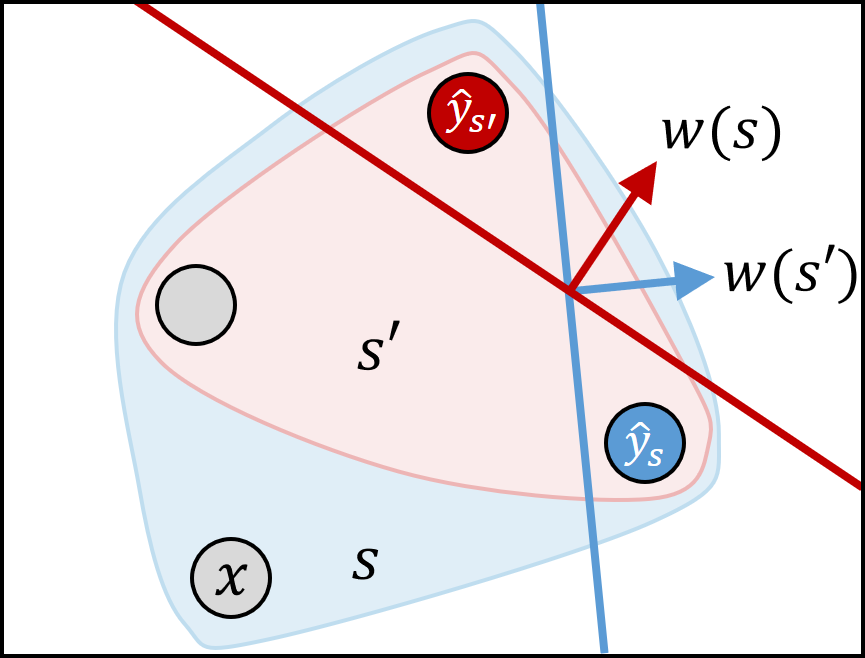}
	\caption{\textbf{\SDW}: $\inner{w(s),\f(x)}$
		\label{fig:viol_viz_sdw}}
\end{subfigure}
\,\,
\begin{subfigure}[b]{0.32\linewidth}
	\centering\includegraphics[width=1\linewidth]{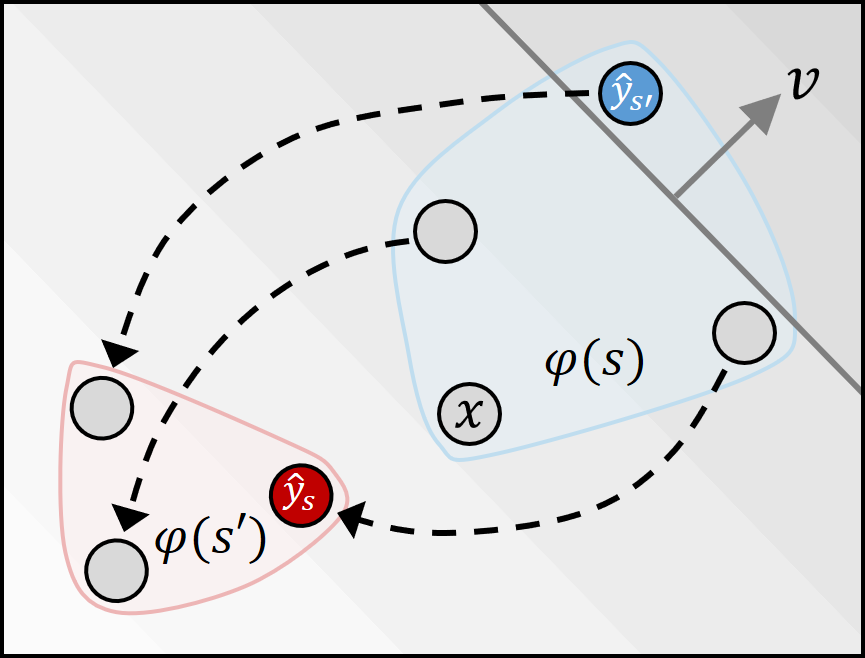}
	\caption{\textbf{\SDE}: $\inner{v,\feat(x,s)}$
		\label{fig:viol_viz_sde}}
\end{subfigure}%
		\caption{An illustration of how set-dependent aggregation can violate IIA.
			Points represent an embedding of items in $\R^\ell$ with $\ell=2$.
			Two sets are shown: $s$ and $s' = s \setminus \{x\}$ for some $x \in s$.
			IIA dictates that removing $x$ from $s$ 
			does not change the prediction $\yhat$.
			This is the case for linear aggregation (\subref{fig:viol_viz_lin}),
			where neither weights $v \in \R^\ell$ or embedding mapping $\feat$ depend on the set.
			When weights $\w$  are a function of $s$, removing an item
			can change the scoring direction (\subref{fig:viol_viz_sdw}).
			When $\feat$ is a function of $s$, removing an item
			can change the spatial position of items (\subref{fig:viol_viz_sde}).
			Both of these allow $s$ and $s'$ to have different maximizing items,
			and hence different predictions.
		}
		\label{fig:viol_viz}
\end{figure}


At the core of our approach is the idea of \emph{aggregation}:
combining a collection of item score functions $\f = (f_1,\dots,f_\ell)$
into a single score function $\agg$.
We refer to $\ell$ as the aggregation dimension,
and with slight abuse of notation use $\f(x)$ to denote the vector function $\f(x)=(f_1(x),\dots,f_\ell(x))$.
Aggregation is a popular approach to enriching the hypothesis class, 
but it is easy to show that standard forms of aggregation
(e.g., linear aggregation $\inner{v,\f(x)}$, $v \in \R^\ell$ as in bagging or boosting)
are still bound to IIA.
Hence, to express violations, aggregation must be \emph{set-dependent}.
Our model is a combination of two powerful set-dependent aggregation mechanisms
(illustrated in Fig. \ref{fig:viol_viz}):
\vspace{0.1cm}

\textbf{Set-dependent weights}:
Let $s \in \allsets$.
Starting with $g(x;v) = \inner{v,\f(x)}$, 
consider replacing each scalar weight $v_i$ with a set function
$w_i( \cdotp) : \allsets \rightarrow \R$. This gives the following aggregation model:
\beq
\text{\textbf{\SDW:}} \quad
\agg(x,s;w) = \inner{w(s), \f(x)}, \qquad
w(s) = (w_1(s), \dots, w_\ell(s))
\label{eq:agg_w}
\eeq
Denote $\yhat = \argmax_{x \in s} \inner{w(s),\f(x)}$,
and let $s' = s \setminus \{x\}$ for some $x \neq \yhat$.
Recall that IIA is violated when removing $x$ from $s$ changes the prediction.
Here, because removing $x$ changes the input to each $w_i$, 
the contribution of each $f_i \in \f$ to $\agg$ changes from $w_i(s)$ to $w_i(s')$.
Geometrically, the scoring direction changes,
and therefore the choice prediction can change as well (Fig. \ref{fig:viol_viz_sdw}).
\vspace{0.15cm}

\textbf{Set-dependent embeddings}:
Now, consider instead replacing each item score function $f_i \in \f$
with an item-centric set function $\feat_i(\cdotp,\cdotp) : \X \times \allsets \rightarrow \R$. The aggregation model is now:
\beq
\text{\textbf{\SDE:}} \quad
\agg(x,s;v,\feat) = \inner{v, \feat(x,s)}, \qquad
\feat(x,s) = (\feat_1(x,s), \dots, \feat_\ell(x,s))
\label{eq:agg_emb}
\eeq
The vector $\feat(x,s)$ is a set-dependent embedding of $x$
(the dependence on $\f$ is implicit).
The predicted item is that whose embedding is closest to $v$'s. 
Here, when the set changes (i.e., by removing an item),
items shift in the embedded space,
and the prediction can change (Fig. \ref{fig:viol_viz_sde}).
\vspace{0.15cm}

Our final and most general \textbf{set-dependent aggregation} model combines both mechanisms:
\beq
\text{\textbf{SDA}:} \quad
\agg(x,s;w,\feat) = \inner{w(s), \feat(x,s)}
\label{eq:agg_comb}
\eeq


\subsection{Inductive bias}
To allow for efficient learning, we implement $w(s)$ and $\feat(x,s)$
using key principles from behavioral choice theory,
suggesting that the choices are based on a relative, a-symmetric,
and context-dependent perception of utility.
In accordance, we instantiate:
\beq
w(s) = w(\f(s)), \qquad
\feat(x,s) = \nonlin \big(\f(x) - \r(\f(s)) \big)
\label{eq:agg_ib}
\eeq
where $\f(s)=\{\f(x)\}_{x \in s}$.
Here, $\f$ models item utilities,
$r$ is a set-specific reference point for comparing utilities
\cite{tversky1991loss},
$\nonlin$ is a loss-averse evaluation mechanism
\cite{kahneman1979prospect,tversky1992advances},
and $w$ integrates multiple evaluations in a context-dependent manner
\cite{tversky1969intransitivity,tversky1993context}.
This construction generalizes many models from the fields of economics, psychology,
and marketing:
\begin{claim}
	The aggregation model in Eq. \eqref{eq:agg_ib}
	generalizes choice models in
	\cite{tversky1969intransitivity,mcfadden1973conditional,mcfadden1978modeling,
		kaneko1979nash,kalai2002rationalizing,kivetz2004alternative,orhun2009optimal}.
	
	\label{lemma:multiself}
\end{claim}
 See Appendix \ref{app:ib} for details. In practice, we implement $w$ and $r$ using appropriate neural networks,
and $\nonlin$ using an a-symmetric s-shaped nonlinearity \cite{maggi2004characterization}.
When $\F$ includes differentiable functions (i.e., linear functions or neural networks),
$\agg$ becomes differentiable, and optimization can be done using standard gradient methods.



\section{Theoretical Analysis} \label{sec:theory}

The complexity of $\Agg$ is controlled by three elements:
the base class $\F$, the set functions $w$ and $\feat$, and the dimension $\ell$.
In the next sections we consider how these effect learning.

The goal of learning is to find some $\agg \in \Agg$ that minimizes the 
\emph{expected risk}:
\beq
\label{eq:exp_risk}
\err{}{}(\agg) = \expect{\dist}{\loss(y,h_\agg(s))}
\eeq
where $\loss(y,\yhat)=\one{y\neq\yhat}$ is the 0/1 loss,
and predictions follow the decision rule $h_\agg(s)=\argmax_{x \in s} g(x|s)$.
In practice, learning typically involves minimizing the \emph{empirical risk}
over the sample set $\smplset=\{(s_i,y_i)\}_{i=1}^m$:
\beq
\err{}{\smplset}(\agg) = \sum_{i=1}^m \loss(y_i,h_\agg(s_i)) 
\label{eq:emp_risk}
\eeq
possibly under some form of regularization.
A good function class is therefore one that balances between being able to
fit the data well \emph{in principle}
(i.e., having a low optimal $\err{}{}(\agg)$)
and \emph{in practice}
(i.e., having a low optimal $\err{}{\smplset}(g)$ and a guarantee on its distance from
the corresponding $\err{}{}(g)$).
This can be seen by decomposing the expected risk into two error types---\emph{approximation error} and \emph{estimation error}:
\beq
\err{}{}(\agg) = \underbrace{\err{}{}(\agg^*)}_{\text{approx.}} + 
\underbrace{\err{}{}(\agg) - \err{}{}(\agg^*)}_{\text{estimation}},\qquad
\agg^* = \argmin_{\agg' \in \Agg} \err{}{}(\agg') \nonumber
\eeq
In this section we bound both types of errors.
For approximation error, which considers the best achievable error,
we show how the capacity of functions in $\Agg$ to relax IIA grows as $\ell$ increases.
For estimation error, which is typically controlled by the generalization error
$\err{}{\smplset}(\agg)-\err{}{}(\agg)$ \cite{shalev2014understanding},
we give Rademacher-based generalization bounds establishing the learnability of aggregators.


\subsection{Approximation Error} \label{sec:apx_err}


Our first result shows how aggregation combines predictors that \emph{cannot}
express violations into a model that can account for \emph{any} form of violation.
Our main theorem reveals how this is achieved:
aggregation uses score functions to `isolate' regions of violating sets,
and within each region, operate independently.
Technically, this is shown by decomposing the approximation error of $\Agg$
over these regions, where the error in each region depends on the error of $\F$.

We begin with some definitions.
Recall that IIA mandates that if 
$y$ is chosen from $s$,
then it should also be chosen from any $s' \subseteq s$ with $y \in s'$.
This motivates a definition of IIA violation due to \cite{ambrus2015rationalising}:\footnote{
	The definition in \cite{ambrus2015rationalising} additionally requires
	$s$ to be of maximal size, which they need for a counting argument.
}
\begin{definition}[IIA Violation, \cite{ambrus2015rationalising}] \label{def:viol}
	A set $s$ \textbf{violates IIA} if
	there exists $s' \supset s$ with $c(s) \neq c(s') \in s$.
\end{definition}
Denote the set of all violating sets by $\Viol$
and non-violating sets by $\Nonviol = \allsets \setminus \Viol$. 
Let the predicate $s \succarg{f} s'$ be true if
$f$ scores items in $s$ higher than items in $s'$,
i.e., $f(x) > f(x')\,\,\forall x \in s, x' \in s'$. 
\begin{definition}[Separation]
	A score function $f$ \textbf{separates} $s$ if $s \succarg{f} \blue{\X \setminus s}$, 
	defining the \textbf{separable region}:
	$$\sep_f = \{ s \in \Viol \,:\, s \text{ is separated by } f \}$$
	\label{def:separation}
\end{definition} 

We are now ready to state our main result
(see Appendix \ref{app:approx} for proof and further details).
For some $A \subseteq \allsets$, denote
$\err{*}{}(\F|A)=\min_{f \in \F}\expect{\dist}{\loss(y,h_{f}(s))|s \in A}$
and $p_A = Pr_\dist(s \in A)$..
\begin{theorem} 
	The approximation error $\err{*}{}(\Agg)$ decomposes over:
	\begin{enumerate}
		\item
		\textbf{non-violating sets} with risk at most
		$\err{*}{}(\F|\Nonviol)$
		\item
		\textbf{separable violating regions} with risk at most
		$\err{*}{}(\F|\sep_f)$ per region 
	\end{enumerate}
	and the set of separable regions is optimal under budget constraint $\ell$.
	Specifically, we have:
	\beq
	\err{*}{}(\Agg) \le
	p_\Nonviol  \, \err{*}{}(\F|\Nonviol) + 
	\sum_{f \in \f} p_{\sep_f}  \, \err{*}{}(\F|\sep_f) + 
	p_{\Viol \setminus \sep_\f}
	\label{eq:err_decomp}
	\eeq
	for the optimal $\f \in \F^{\ell'}$,
	where 
	$\sep_\f = \cup_{f \in \f} \sep_f$,
	$\ell' \le (\ell-1)/5$.
	\label{thm:err_decomp}
\end{theorem}
Theorem \ref{thm:err_decomp} 
hints at how aggregation can account for violations:
it partitions $\allsets$ into regions,
and applies $\F$ to each region independently.
In this process, $\F$ plays a corresponding dual role
of both determining separable regions $\sep_f$
and predicting within these regions.
This demonstrates how the \emph{type} of accountable violations is controlled
by the expressivity of $\F$, 
and how the \emph{number} of violations (via the number of regions) is controlled by $\ell$.
In principal, aggregation can account for any violation:
\begin{corollary}
	If $y=c(s)$ for some choice function $c$ and for all $s \in \allsets$,
	then with sufficiently large $\ell$ and sufficiently expressive $\F$,
	the approximation error vanishes, i.e., $\err{*}{}(\Agg)=0$.\footnote{
		Here we consider the excess error, i.e., without the irreducible optimal Bayes error.
		\qquad\qquad\qquad\qquad}
\end{corollary}

Our proof takes the main building blocks of \cite{ambrus2015rationalising}---
objects called ``triple bases''
---and carefully weaves them within a statistical framework.
As in \cite{ambrus2015rationalising}, results apply to classes 
satisfying a natural property of \emph{scale-invariance};
these include all but one model from Claim \ref{lemma:multiself}.
The main lemma in the proof requires that $\Agg$ is defined over an auxiliary
neural-network class of score functions that is slightly richer than $\F$,
and shows how ``implementing'' triple bases using simple neural circuits
enables a decomposition of the error over violating regions.

The general statement of Theorem \ref{thm:err_decomp} is that
the more complex $\F$ and the larger $\ell$,
the better the approximation error.
Our next results quantifies how this trades off with estimation error.




\subsection{Estimation Error} \label{sec:est_err}
Our next results establish the learnability of aggregation.
We give sample complexity bounds on the empirical risk
$\err{}{\smplset}(\agg) = \expect{\smplset}{\loss(y,\agg(s))}$
showing how learning both weight and embedding mechanisms (Sec. \ref{sec:model})
depends on $\F$, $\ell$, and the aggregation components in $\agg$.
Proofs are based on Rademacher bounds and are given in Appendix \ref{app:est}.
To simplify the analysis,
we focus on set operations and on a linear base class:
\[
\F_\text{lin}^\rho = \{ \f(x) = x^\top \Theta \,:\, \Theta \in \R^{d \times \ell}, \,
\|\Theta\|_\rho \le 1 \}
\]
where $\|\cdotp\|_\rho$ is the induced $\rho$-norm.
This suffices to cover the models from Claim \ref{lemma:multiself}.\footnote{
	All except for one of the three models proposed in \cite{kivetz2004alternative}.}

For any function $q$, we denote its Lipschitz constant by $\lambda_q$
when it is scalar-valued and by $\lambda_q^\rho$ when it is vector valued and with respect to
the $\rho$-norm.
We also use $X_\infty = \max_{x \in \X} \|x\|_\infty$.

The first bound applies to the set-dependent weight mechanism (Eq. \eqref{eq:agg_w}).
\begin{theorem}
	Let $\Agg$ be a class of aggregators over $\F_\text{lin}^\infty$
	of the form $g(x,s) = \inner{\w(s), \f(x)}$.
	Then for all $\dist$ and any $\delta \in [0,1]$, it holds that
	\begin{equation*}
		\forall \, \agg \in \Agg, \quad
	\err{}{}(\agg) \le \err{}{\smplset}(\agg) + 
	4 X_\infty^2 \lambda_w^\rho \sqrt{\frac{2 \log 2d}{m}} +
	O \left( \sqrt{\frac{ \log(1 / \delta)}{m}} \right)
	\end{equation*}
	with probability of at least $1-\delta$.
\end{theorem}
The second bound applies to the set-dependent embedding  mechanism (Eq. \eqref{eq:agg_emb}),
where for concreteness we set $\feat$ according to the inductive bias model
in Eq. \eqref{eq:agg_ib}.
\begin{theorem} 
	Let $\Agg$ be a class of aggregators over $\F_\text{lin}^1$
	of the form $g(x,s) = \inner{v,\feat(x,s)}$ 
	where $\feat(x,s)=\nonlin(\f(x)-r(s))$ as in Eq. \eqref{eq:agg_ib}.
	Then for all $\dist$ and any $\delta \in [0,1]$, it holds that
	\begin{equation*}
	\forall \, \agg \in \Agg, \quad
	\err{}{}(\agg) \le \err{}{\smplset}(\agg) + 
	4 W_1 X_\infty \lambda_\mu (1+\lambda_r^\rho) \sqrt{\frac{2 \log 2d}{m}}
	+ O \left( \sqrt{\frac{\log(1 / \delta)}{m}}\right)
	\end{equation*}
	with probability of at least $1-\delta$, where $W_1 = \|w\|_1$.
\end{theorem}
Both results show what governs learnability:
the dimension $\ell$ plays an implicit role
via the Lipschitz constants of the set operations,
and $\F$ effects the bound via the norm of $w$
(also depending on $\ell$).

\section{Experiments} \label{sec:experiments}

We now present our experimental evaluation on click prediction tasks on online platforms.


\begin{figure}[!t]
	\begin{center}
		\includegraphics[width=0.8\columnwidth,trim=4cm 0 0 0,clip]{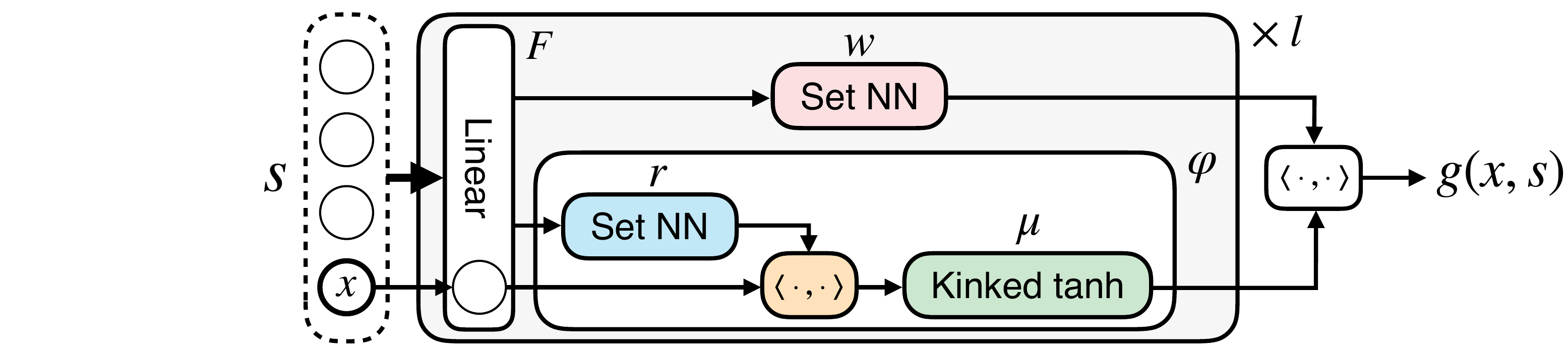} 
		\caption{The proposed aggregation model}
		\label{fig:diagram}
	\end{center}
\end{figure}

\textbf{Datasets.} We evaluate our method on three large datasets:
flight itineraries
from Amadeus\footnote{See \cite{mottini2017deep}},
hotel reservations
from Expedia\footnote{\raggedright\url{www.kaggle.com/c/expedia-personalized-sort}}, 
and news recommendations
from Outbrain\footnote{\raggedright\url{www.kaggle.com/c/outbrain-click-prediction}}.
Each dataset includes sets of alternatives and the corresponding user choices. We focus on examples where users clicked on exactly one item.
Features describe items (e.g., price, quality, or category)
and context (e.g., query, date and time).
Unfortunately, for reasons of privacy, very little user information is available. Appendix \ref{app:datasets} includes further details.

%
%
%
%
%

\textbf{Baselines.} 
We compare against three baseline categories:
methods that satisfy IIA,
methods that target specific violations of IIA,
and methods that do not satisfy IIA but also do not directly model violations.
This lets us explore whether capturing IIA violations is useful,
and if so, what form of deviation is preferable.
We focus on approaches that can be applied to
large-scale data as the above.

For IIA methods, we use Multinomial Logit (MNL)
\cite{mcfadden1973conditional,train2009discrete},
SVMRank \cite{joachims2006training},
and RankNet \cite{burges2005learning}.
These methods predict based on item-centric score functions,
but differ in the way they are optimized.
For non-IIA methods,
we use a discrete Mixed MNL model \cite{train2009discrete},
AdaRank \cite{adarank},
and Deep Sets \cite{zaheer2017deep}.
These differ in how they consider item dependencies:
Mixed MNL models user populations, 
ListNet captures set dependencies via the loss function,
and Deep Sets attempts to universally approximate general set functions.
We also compare to simple baselines based on price or quality, and to random.

\textbf{Aggregation models.}
For our approach, 
we evaluate the performance of three models,
each corresponding to one of the aggregation mechanisms from Sec. \ref{sec:model}.\\
\vspace{0.1cm}
\storestyleof{itemize}
\begin{listliketab}
	\begin{tabular}{Llll}
		\textbullet &  
		\textbf{Set-dependent weights}: &
		 $\agg(x,s) = \inner{w(s), \f(x)}$  & 
	  [\SDW, Eq. \eqref{eq:agg_w}]   \\
		\textbullet & 
		\textbf{Set-dependent embeddings}: &
		$\agg(x,s) = \inner{v, \feat(x,s)}$  & 
		[\SDE, Eq. \eqref{eq:agg_emb}]  \\
		\textbullet & 
		\textbf{Combined inductive bias model}:&
		$\agg(x,s) = \inner{w(s), \feat(x,s)}$ & 
		[\SDA, Eq. \eqref{eq:agg_comb}]
	\end{tabular}
\end{listliketab}

\vspace{-0.4cm}
The set function $w$ is implemented by a small
permutation-invariant neural network \cite{zaheer2017deep}
having 2 hidden layers of size 16 with tanh activations and use mean pooling.
In the combined model, we use a slight generalization 
of Eq. \eqref{eq:agg_comb},
letting $\feat(x,s)=\nonlin (\inner{\f(x),\r(s)} )$
where $\f$ and $r$ are now vector-valued, and are compared by an inner product.
For the a-symmetric sigmoidal $\nonlin$ we use a `kinked' tanh,
and $r$ has the same architecture as $w$ but outputs a vector. 
Figure \ref{fig:diagram} illustrates this architecture.
As noted, these design choices follow key principles from behavioral choice theory.
Appendix \ref{app:ablation} includes further details and ablation studies.

To highlight the contribution of aggregation,
for all models we use a simple linear base class $\F$,
allowing a clean differential comparison to IIA methods such as MNL
(linear score functions, no set operations)
and to non-IIA methods such as DeepSets
(neural set functions, no linear components).
Our main results use $\ell=24$, which strikes a good balance between
performance and runtime
(in general, accuracy increases with $\ell$, but most of the gain is already reached at $\ell=4$; see Fig. \ref{fig:extra} left).


\textbf{Setup.}
Results are based on averaging 10 random 50:25:25 train-validation-test splits.
The methods we consider are varied in their expressive power and computational requirements.
Hence, for a fair comparison (and to allow reasonable run times), 
each trial includes 10,000 randomly sampled examples.
Performance is measured using top-1 accuracy, top-5 accuracy, and mean reciprocal rank (MRR).
For all methods we tuned regularization, dropout, and learning rate (when applicable)
using Bayesian optimization. 
Default parameters were used for other hyper-parameters.
For optimization we used Adam with step-wise exponential decay.
See Appendix \ref{app:setup} for further details.



\begin{table}
	\centering
	\tblsmall
	\setlength{\tabcolsep}{2pt}
	{\setlength{\tabcolsep}{0.4em}
	\begin{tabular}{rrlccccccccc}
		&   &   & \multicolumn{3}{c}{\textbf{Amadeus}} & \multicolumn{3}{c}{\textbf{Expedia}} & \multicolumn{3}{c}{\textbf{Outbrain}} \\
		\cmidrule[3\cmidrulewidth]{4-12}  &   &   & Top-1 & Top-5 & MRR & Top-1 & Top-5 & MRR & Top-1 & Top-5 & MRR \\
		\cmidrule{3-12}\multicolumn{1}{r}{\multirow{3}[2]{*}{\begin{sideways}\textbf{Ours}\end{sideways}}} &   & \textbf{SDA} & \textbf{45.42}\tbltiny{ $\pm 0.5$} & \textbf{93.37}\tbltiny{ $\pm 0.0$} & \textbf{2.31}\tbltiny{ $\pm 0.3$} & \textbf{31.49}\tbltiny{ $\pm 0.2$} & \textbf{86.91}\tbltiny{ $\pm 0.2$} & \textbf{2.99}\tbltiny{ $\pm 0.0$} & \textbf{38.04}\tbltiny{ $\pm 0.3$} & \textbf{94.54}\tbltiny{ $\pm 0.1$} & \textbf{2.42}\tbltiny{ $\pm 0.0$} \\
		&   & \textbf{SDE} & 39.62\tbltiny{ $\pm 0.4$} & 91.70\tbltiny{ $\pm 0.3$} & 2.52\tbltiny{ $\pm 0.0$} & 31.47\tbltiny{ $\pm 0.2$} & 86.52\tbltiny{ $\pm 0.2$} & 3.00\tbltiny{ $\pm 0.0$} & 37.59\tbltiny{ $\pm 0.3$} & 94.26\tbltiny{ $\pm 0.1$} & 2.44\tbltiny{ $\pm 0.0$} \\
		&   & \textbf{SDW} & 39.98\tbltiny{ $\pm 0.4$} & 91.89\tbltiny{ $\pm 0.3$} & 2.50\tbltiny{ $\pm 0.0$} & 31.27\tbltiny{ $\pm 0.2$} & 86.67\tbltiny{ $\pm 0.2$} & 2.99\tbltiny{ $\pm 0.0$} & 37.86\tbltiny{ $\pm 0.3$} & 94.48\tbltiny{ $\pm 0.2$} & 2.43\tbltiny{ $\pm 0.0$} \\
		\cmidrule{3-12}\multicolumn{1}{r}{\multirow{3}[2]{*}{\begin{sideways}IIA\end{sideways}}} &   & MNL \cite{mcfadden1973conditional} & 38.42\tbltiny{ $\pm 0.5$} & 91.02\tbltiny{ $\pm 0.3$} & 2.57\tbltiny{ $\pm 0.0$} & 30.06\tbltiny{ $\pm 0.2$} & 86.34\tbltiny{ $\pm 0.1$} & 3.07\tbltiny{ $\pm 0.0$} & 37.74\tbltiny{ $\pm 0.3$} & 94.52\tbltiny{ $\pm 0.2$} & 2.43\tbltiny{ $\pm 0.0$} \\
		&   & SVMRank \cite{joachims2006training} & 40.27\tbltiny{ $\pm 0.4$} & 91.94\tbltiny{ $\pm 0.3$} & 2.49\tbltiny{ $\pm 0.0$} & 31.28\tbltiny{ $\pm 0.2$} & 86.24\tbltiny{ $\pm 0.1$} & 3.01\tbltiny{ $\pm 0.0$} & 37.68\tbltiny{ $\pm 0.3$} & 94.46\tbltiny{ $\pm 0.1$} & 2.43\tbltiny{ $\pm 0.0$} \\
		&   & RankNet \cite{burges2005learning} & 37.44\tbltiny{ $\pm 0.7$} & 84.67\tbltiny{ $\pm 1.7$} & 3.02\tbltiny{ $\pm 0.1$} & 23.82\tbltiny{ $\pm 0.5$} & 81.85\tbltiny{ $\pm 0.6$} & 3.43\tbltiny{ $\pm 0.0$} & 35.32\tbltiny{ $\pm 0.8$} & 91.55\tbltiny{ $\pm 0.7$} & 2.65\tbltiny{ $\pm 0.1$} \\
		\cmidrule{3-12}\multicolumn{1}{r}{\multirow{3}[2]{*}{\begin{sideways}non-IIA\end{sideways}}} &   & Mixed MNL \cite{train2009discrete} & 37.96\tbltiny{ $\pm 0.3$} & 90.40\tbltiny{ $\pm 0.3$} & 2.62\tbltiny{ $\pm 0.0$} & 27.28\tbltiny{ $\pm 0.6$} & 84.24\tbltiny{ $\pm 0.3$} & 3.22\tbltiny{ $\pm 0.0$} & 37.72\tbltiny{ $\pm 0.3$} & 94.42\tbltiny{ $\pm 0.1$} & 2.43\tbltiny{ $\pm 0.0$} \\
		&   & AdaRank \cite{adarank} & 37.27\tbltiny{ $\pm 0.4$} & 72.34\tbltiny{ $\pm 0.3$} & 4.03\tbltiny{ $\pm 0.0$} & 26.70\tbltiny{ $\pm 0.2$} & 83.21\tbltiny{ $\pm 0.2$} & 3.29\tbltiny{ $\pm 0.0$} & 37.47\tbltiny{ $\pm 0.3$} & 94.40\tbltiny{ $\pm 0.2$} & 2.44\tbltiny{ $\pm 0.0$} \\
		&   & Deep Sets \cite{zaheer2017deep} & 40.36\tbltiny{ $\pm 0.5$} & 91.92\tbltiny{ $\pm 0.3$} & 2.48\tbltiny{ $\pm 0.0$} & 29.87\tbltiny{ $\pm 0.3$} & 86.26\tbltiny{ $\pm 0.2$} & 3.06\tbltiny{ $\pm 0.0$} & 37.51\tbltiny{ $\pm 0.3$} & 94.30\tbltiny{ $\pm 0.1$} & 2.44\tbltiny{ $\pm 0.0$} \\
		\cmidrule{3-12}\multicolumn{1}{r}{\multirow{2}[2]{*}{\begin{sideways}basic\end{sideways}}} &   & Price/Quality & 36.44\tbltiny{ $\pm 0.3$} & 87.23\tbltiny{ $\pm 0.2$} & 2.79\tbltiny{ $\pm 0.0$} & 17.92\tbltiny{ $\pm 0.1$} & 77.67\tbltiny{ $\pm 0.1$} & 3.79\tbltiny{ $\pm 0.0$} & 24.17\tbltiny{ $\pm 0.1$} & 25.08\tbltiny{ $\pm 0.1$} & 8.13\tbltiny{ $\pm 0.0$} \\
		&   & Random & 25.15\tbltiny{ $\pm 0.5$} & 32.87\tbltiny{ $\pm 0.6$} & 6.49\tbltiny{ $\pm 0.0$} & 14.13\tbltiny{ $\pm 0.1$} & 32.35\tbltiny{ $\pm 0.2$} & 6.38\tbltiny{ $\pm 0.0$} & 22.21\tbltiny{ $\pm 0.1$} & 23.10\tbltiny{ $\pm 0.1$} & 8.32\tbltiny{ $\pm 0.0$} \\
		\cmidrule[3\cmidrulewidth]{3-12}
	\end{tabular}%
	}
	\label{tbl:results}
	\vspace{0.2cm}
\caption{Main results. Values are averaged over 10 random splits (standard errors in small font).}
\end{table}

\subsection{Results}
Our main results are presented in Table 1,
which compares the different methods for choice sets up to 10, 10, 12 items for each dataset (the full results for a range of maximal number of items can be found 
in Appendix \ref{app:fulltbl}).
As can bee seen, SDA outperforms other baselines in all settings.

When comparing predictive top-1 accuracy,
non-IIA methods tend to outperform IIA methods by a margin.
Discrete Mixed MNL, which deviates from IIA by targeting certain inter-item dependencies,
is observed as the most competitive baseline.
DeepSets, which is our most general non-IIA baseline, shows mixed results,
and relatively high standard errors.
Deepsets portrays the invariances required for expressing violations with significantly more parameters than SDA (Appendix \ref{app:baselines}),
but lacks the inductive bias that SDA holds, thus demonstrating the importance of the latter.

For measures that extend beyond the first item (such as top-5 and MRR),
IIA ranking methods perform well, likely because their objective consider
relations between items.
Nonetheless, they are still outperformed by SDA, even though its
objective considers only the top item.
We conjecture that training SDA with a ranking loss will increase performance further.



\subsection{Analysis}

\textbf{Accuracy and violation capacity}.
Figure \ref{fig:extra} (left) presents the accuracy of an aggregator for increasing  $\ell$. Results show that accuracy steadily increases,
although roughly $90\%$ of the gain is accuracy is achieved by $\ell=4$.
Theorem \ref{thm:err_decomp} suggests that increasing $\ell$ helps
by covering additional regions of violating sets.
To empirically quantify this, we measure \emph{violation capacity}:
\beq
\vcap_\smplset(\agg) = \frac{1}{|\smplset|}\sum_{(s,y) \in \smplset}
\frac{1}{|s|-1} \sum_{y \neq x \in s} \1{g(x,s_{-x}) \neq y}
\label{eq:violation_capacity}
\eeq
where $s_{-x} = s \setminus \{x\}$.
Violation capacity measures how frequently the prediction changes
when single items are removed from the choice set,
and is a first-order approximation of the violation frequency captured by a model.
Figure \ref{fig:extra} (left) reveals a tight correlation
between accuracy and $\vcap$.

\textbf{Violation budget allocation}.
The decomposition in Theorem \ref{thm:err_decomp} shows how aggregation
allocates one score function to cover all non-violating sets $\Nonviol$,
and uses the remaining $\ell-1$ score functions for targeting and handling violating regions.
Empirically, this suggests that $\vcap$ should vary across choice sets:
high $\vcap$ for targeted violating regions, low $\vcap$ for $\Nonviol$ and all other regions.
Figure \ref{fig:extra} (right) compares the violation capacity of \SDA\, MNL,
and Deep Sets. For each method, the diagram shows average $\vcap$ values
partitioned according to intersecting correctness regions
(i.e., examples that only \SDA\ was correct on,
examples that \SDA\ and MNL were correct on but not Deep Sets, etc.).
The results highlight how \SDA\ efficiently allocates its violation budget,
executing very little violation capacity on examples that MNL is correct on,
and allocating the minimally necessary amount to the rest.
In contrast, MNL hos no violation capacity (as it cannot express violations),
and Deep Sets over-utilizes its unconstrained capacity to violate.

\textbf{Comparing aggregators}
Claim \ref{lemma:multiself} states that \SDA\ generalizes many aggregation
models from the literature,
and by preserving their structural form (i.e., Eq. \eqref{eq:agg_ib}),
introduces useful inductive bias.
Figure \ref{fig:extra} (center)
compares the performance of \SDA\ to the models from Claim \ref{lemma:multiself} on Amadeus.
As can be seen, \SDA\ clearly outperforms other models.
Many of these models were designed for mathematical
or behavioral tractability, having simple components with light parameterization.
These results demonstrate the benefit of 
replacing these with highly flexible parametric neural components.

\begin{figure*}[t!]
	\centering
	\includegraphics[width=0.30\columnwidth]{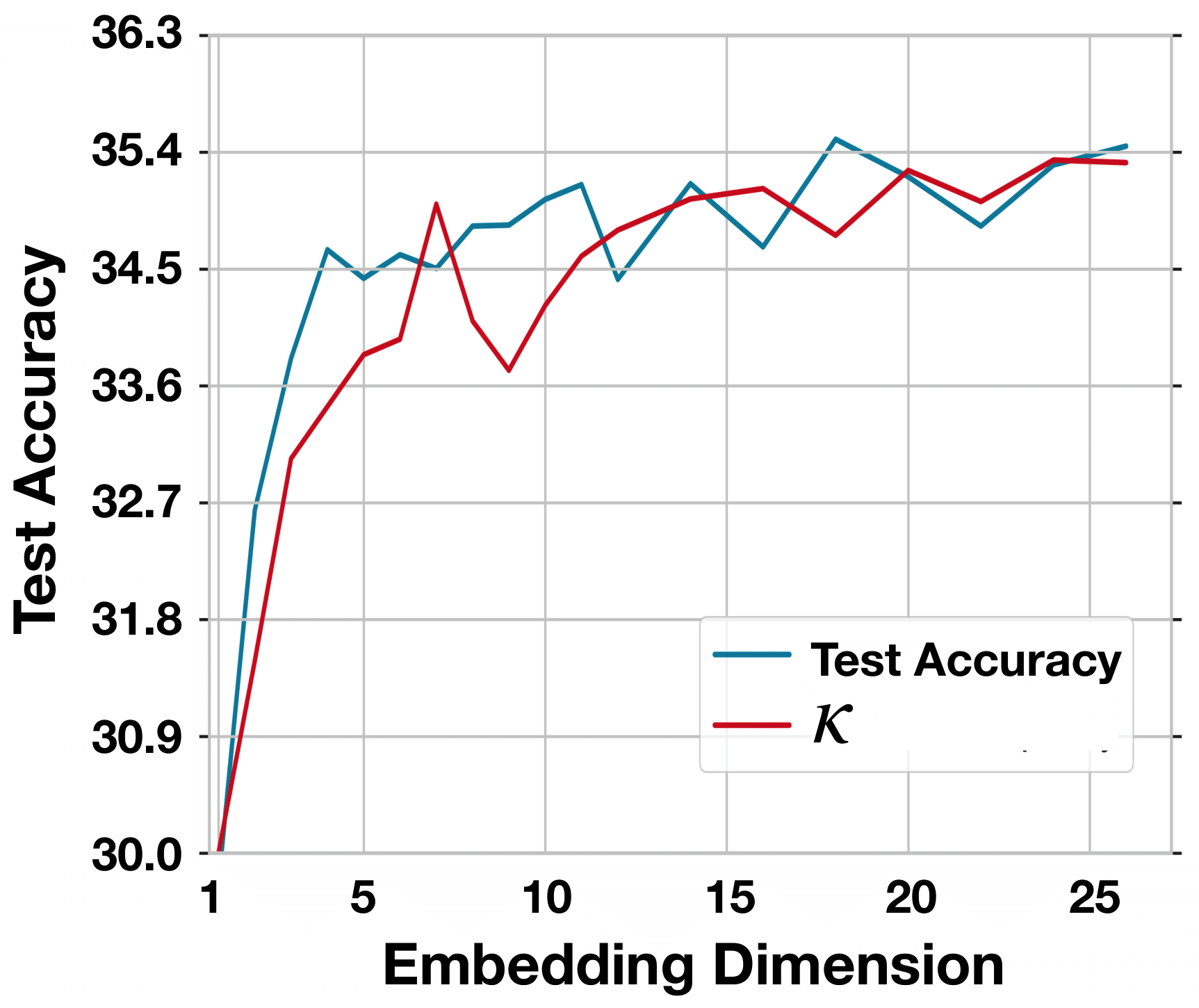}
	\hspace{0.22cm}
	\includegraphics[width=0.33\columnwidth]{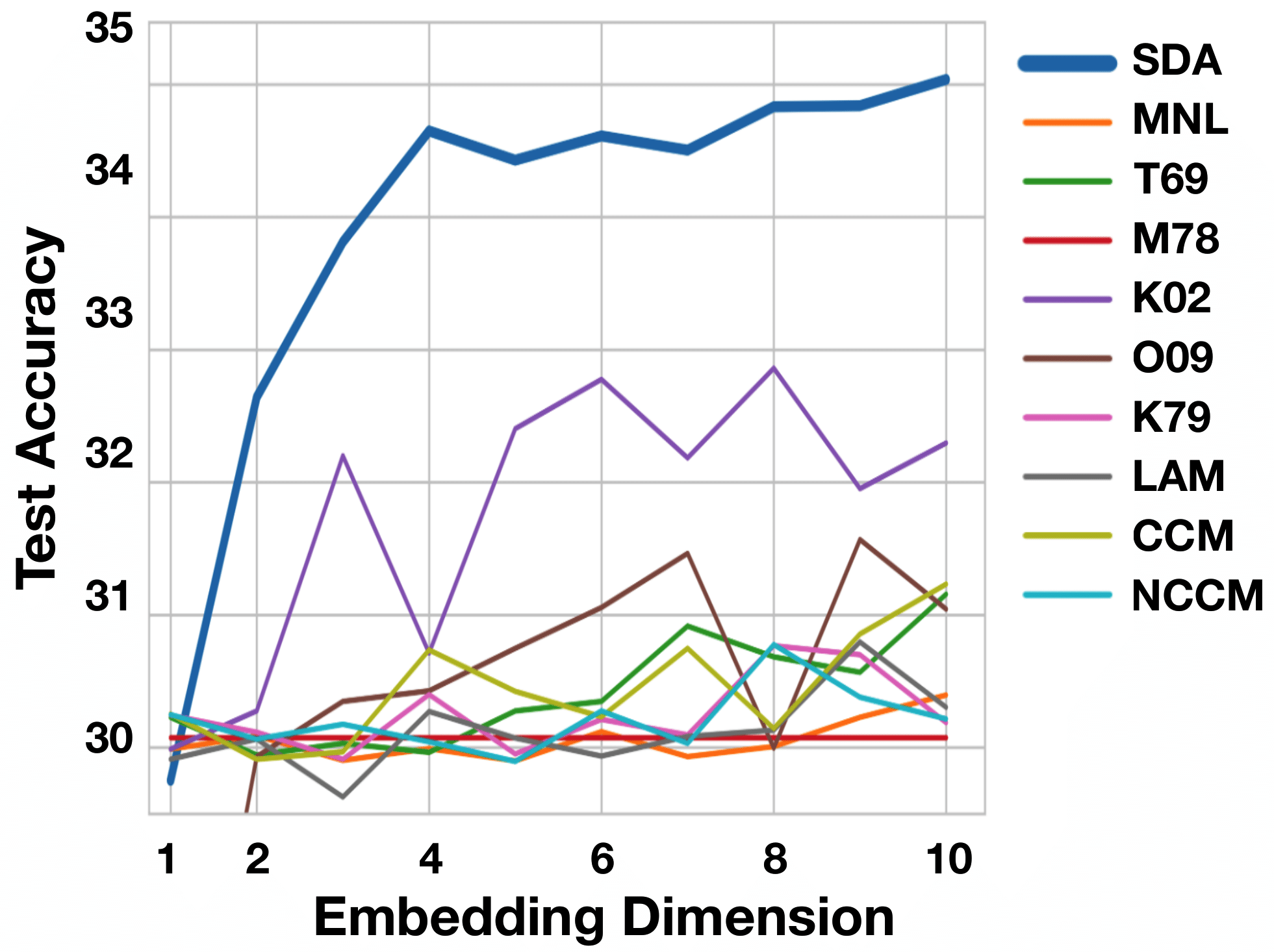}
	\hspace{0.22cm}    
	\includegraphics[width=0.31\columnwidth]{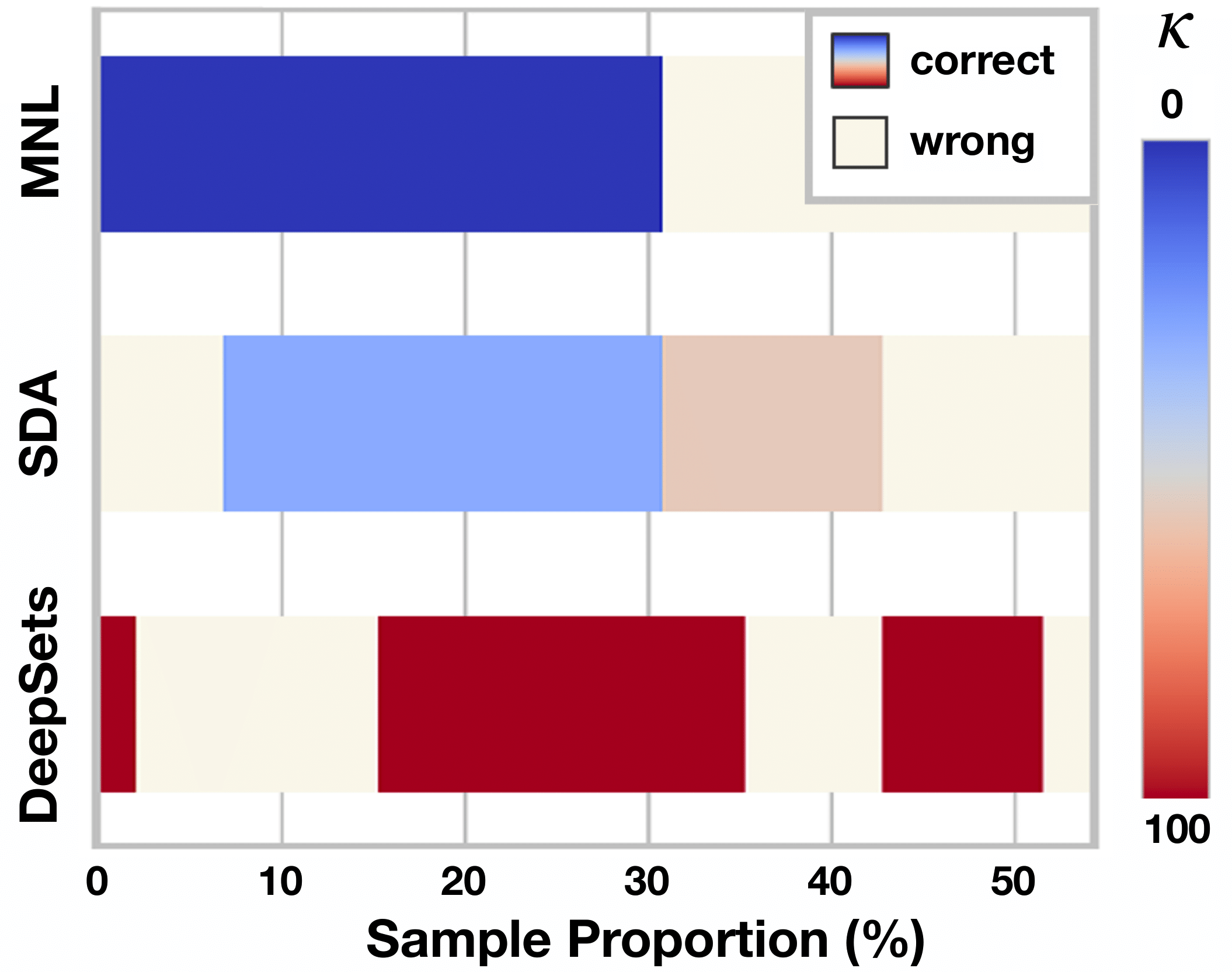}
	\caption{\textbf{Left}:
		Accuracy (blue) and violation c apacity $\vcap$ (red) are
		highly correlated. Most of the accuracy gain is achieved by $\ell=4$.
		\textbf{Center}:
		Accuracy of SDA vs. the choice models it generalizes (Claim \ref{lemma:multiself}),
		showing the benefit of replacing hand-crafted elements with neural components.
		\blue{\textbf{Right}:
		MNL and Deep Sets represent two extremes:
		inability to express violations ($\vcap=0$, dark blue),
		and over-flexibility ($\vcap=1$, dark red).
		SDA finds middle ground by properly allocating its violation `budget',
		focusing mostly on samples that MNL gets wrong (mean $\vcap$ in light red).}
	}
	\label{fig:extra}
\end{figure*}


\section{Conclusions} \label{sec:conclusions}

In this paper we proposed a method for accurately predicting choice.
Human choice follows complex and intricate patterns,
and capturing it requires models that on the one hand are sufficiently expressive,
but on the other are specific enough to be learned efficiently.
Our goal in this paper was to show that aggregation strikes a good balance
between expressivity and specificity, in theory and in practice.

Our work is motivated by the growing need for methods that are accurate and scalable.
Aggregation is a good candidate for two main reasons.
First, aggregation accounts for general violations without imposing any structural assumptions.
This is important since accurately predicting choice is likely impossible within the confines of IIA or of targeted forms of violation.
Second, it provides a simple template into which differential components can be cast.
Aggregators can therefore capitalize on the success of discriminative learning and neural architectures to efficiently optimize predictive accuracy. 

There are two main avenues in which our work can extend.
First, our theoretical results explore the connection between generalization and violation of IIA.
We conjecture that this connection runs deeper,
and that a properly defined ``violation complexity'' can be useful
in giving exact characterizations of learnability.
Second, our work focused on a rudimentary choice task: choosing an item from a set.
There are, however, many other important choice tasks,
such as sequential or subset selection,
posing interesting modeling challenges
which we leave for future work.

\bibliographystyle{plain}
\bibliography{refs}

\begin{thebibliography}{10}

\bibitem{ambrus2015rationalising}
Attila Ambrus and Kareen Rozen.
\newblock Rationalising choice with multi-self models.
\newblock {\em The Economic Journal}, 125(585):1136--1156, 2015.

\bibitem{bartlett2002rademacher}
Peter~L Bartlett and Shahar Mendelson.
\newblock Rademacher and gaussian complexities: Risk bounds and structural
  results.
\newblock {\em Journal of Machine Learning Research}, 3(Nov):463--482, 2002.

\bibitem{benson2016relevance}
Austin~R Benson, Ravi Kumar, and Andrew Tomkins.
\newblock On the relevance of irrelevant alternatives.
\newblock In {\em Proceedings of the 25th International Conference on World
  Wide Web}, pages 963--973. International World Wide Web Conferences Steering
  Committee, 2016.

\bibitem{benson2018discrete}
Austin~R Benson, Ravi Kumar, and Andrew Tomkins.
\newblock A discrete choice model for subset selection.
\newblock In {\em Proceedings of the Eleventh ACM International Conference on
  Web Search and Data Mining}, pages 37--45. ACM, 2018.

\bibitem{burges2005learning}
Chris Burges, Tal Shaked, Erin Renshaw, Ari Lazier, Matt Deeds, Nicole
  Hamilton, and Greg Hullender.
\newblock Learning to rank using gradient descent.
\newblock In {\em Proceedings of the 22nd international conference on Machine
  learning}, pages 89--96. ACM, 2005.

\bibitem{chierichetti2018learning}
Flavio Chierichetti, Ravi Kumar, and Andrew Tomkins.
\newblock Learning a mixture of two multinomial logits.
\newblock In {\em International Conference on Machine Learning}, pages
  960--968, 2018.

\bibitem{fudenberg2006dual}
Drew Fudenberg and David~K Levine.
\newblock A dual-self model of impulse control.
\newblock {\em American economic review}, 96(5):1449--1476, 2006.

\bibitem{green2009choice}
Jerry Green and Daniel Hojman.
\newblock Choice, rationality and welfare measurement.
\newblock 2009.

\bibitem{joachims2006training}
Thorsten Joachims.
\newblock Training linear svms in linear time.
\newblock In {\em Proceedings of the 12th ACM SIGKDD international conference
  on Knowledge discovery and data mining}, pages 217--226. ACM, 2006.

\bibitem{kahneman1979prospect}
Daniel Kahneman and Amos Tversky.
\newblock {Prospect Theory: An Analysis of Decision under Risk}.
\newblock {\em Econometrica}, 47(2):263--291, March 1979.

\bibitem{kalai2002rationalizing}
Gil Kalai, Ariel Rubinstein, and Ran Spiegler.
\newblock Rationalizing choice functions by multiple rationales.
\newblock {\em Econometrica}, 70(6):2481--2488, 2002.

\bibitem{kaneko1979nash}
Mamoru Kaneko and Kenjiro Nakamura.
\newblock The nash social welfare function.
\newblock {\em Econometrica: Journal of the Econometric Society}, pages
  423--435, 1979.

\bibitem{kivetz2004alternative}
Ran Kivetz, Oded Netzer, and V~Srinivasan.
\newblock Alternative models for capturing the compromise effect.
\newblock {\em Journal of marketing research}, 41(3):237--257, 2004.

\bibitem{kleinberg2017human}
Jon Kleinberg, Himabindu Lakkaraju, Jure Leskovec, Jens Ludwig, and Sendhil
  Mullainathan.
\newblock Human decisions and machine predictions.
\newblock {\em The quarterly journal of economics}, 133(1):237--293, 2017.

\bibitem{kleinberg2017comparison}
Jon Kleinberg, Sendhil Mullainathan, and Johan Ugander.
\newblock Comparison-based choices.
\newblock In {\em Proceedings of the 2017 ACM Conference on Economics and
  Computation}, pages 127--144. ACM, 2017.

\bibitem{luce1959individual}
R.~Duncan Luce.
\newblock {\em Individual Choice Behavior: A Theoretical analysis}.
\newblock Wiley, New York, NY, USA, 1959.

\bibitem{maggi2004characterization}
Mario~Alessandro Maggi.
\newblock A characterization of s-shaped utility functions displaying loss
  aversion.
\newblock Technical report, Quaderni di Dipartimento, EPMQ, Universit{\`a}
  degli Studi di Pavia, 2004.

\bibitem{manzini2007sequentially}
Paola Manzini and Marco Mariotti.
\newblock Sequentially rationalizable choice.
\newblock {\em American Economic Review}, 97(5):1824--1839, 2007.

\bibitem{mcfadden1978modeling}
Daniel McFadden.
\newblock Modeling the choice of residential location.
\newblock {\em Transportation Research Record}, (673), 1978.

\bibitem{mcfadden1973conditional}
Daniel McFadden et~al.
\newblock Conditional logit analysis of qualitative choice behavior.
\newblock 1973.

\bibitem{mcfadden2000mixed}
Daniel McFadden and Kenneth Train.
\newblock Mixed mnl models for discrete response.
\newblock {\em Journal of applied Econometrics}, 15(5):447--470, 2000.

\bibitem{mcfadden1977application}
Daniel McFadden, William~B Tye, and Kenneth Train.
\newblock {\em An application of diagnostic tests for the independence from
  irrelevant alternatives property of the multinomial logit model}.
\newblock Institute of Transportation Studies, University of California, 1977.

\bibitem{mottini2017deep}
Alejandro Mottini and Rodrigo Acuna-Agost.
\newblock Deep choice model using pointer networks for airline itinerary
  prediction.
\newblock In {\em Proceedings of the 23rd ACM SIGKDD International Conference
  on Knowledge Discovery and Data Mining}, pages 1575--1583. ACM, 2017.

\bibitem{negahban2018learning}
Sahand Negahban, Sewoong Oh, Kiran~K Thekumparampil, and Jiaming Xu.
\newblock Learning from comparisons and choices.
\newblock {\em The Journal of Machine Learning Research}, 19(1):1478--1572,
  2018.

\bibitem{oh2014learning}
Sewoong Oh and Devavrat Shah.
\newblock Learning mixed multinomial logit model from ordinal data.
\newblock In {\em Advances in Neural Information Processing Systems}, pages
  595--603, 2014.

\bibitem{orhun2009optimal}
A~Ye{\c{s}}im Orhun.
\newblock Optimal product line design when consumers exhibit choice
  set-dependent preferences.
\newblock {\em Marketing Science}, 28(5):868--886, 2009.

\bibitem{osogami2014restricted}
Takayuki Osogami and Makoto Otsuka.
\newblock Restricted boltzmann machines modeling human choice.
\newblock In {\em Advances in Neural Information Processing Systems}, pages
  73--81, 2014.

\bibitem{otsuka2016deep}
Makoto Otsuka and Takayuki Osogami.
\newblock A deep choice model.
\newblock In {\em AAAI}, pages 850--856, 2016.

\bibitem{overgoor2018choosing}
Jan Overgoor, Austin~R Benson, and Johan Ugander.
\newblock Choosing to grow a graph: Modeling network formation as discrete
  choice.
\newblock {\em arXiv preprint arXiv:1811.05008}, 2018.

\bibitem{pfannschmidt2019learning}
Karlson Pfannschmidt, Pritha Gupta, and Eyke H{\"u}llermeier.
\newblock Learning choice functions.
\newblock {\em arXiv preprint arXiv:1901.10860}, 2019.

\bibitem{ragain2016pairwise}
Stephen Ragain and Johan Ugander.
\newblock Pairwise choice markov chains.
\newblock In {\em Advances in Neural Information Processing Systems}, pages
  3198--3206, 2016.

\bibitem{ragain2018choosing}
Stephen Ragain and Johan Ugander.
\newblock Choosing to rank.
\newblock {\em arXiv preprint arXiv:1809.05139}, 2018.

\bibitem{rieskamp2006extending}
J{\"o}rg Rieskamp, Jerome~R Busemeyer, and Barbara~A Mellers.
\newblock Extending the bounds of rationality: Evidence and theories of
  preferential choice.
\newblock {\em Journal of Economic Literature}, 44(3):631--661, 2006.

\bibitem{seshadri2019discovering}
Arjun Seshadri, Alexander Peysakhovich, and Johan Ugander.
\newblock Discovering context effects from raw choice data.
\newblock {\em arXiv preprint arXiv:1902.03266}, 2019.

\bibitem{shafir1993reason}
Eldar Shafir, Itamar Simonson, and Amos Tversky.
\newblock Reason-based choice.
\newblock {\em Cognition}, 49(1-2):11--36, 1993.

\bibitem{shah2017simple}
Nihar~B Shah and Martin~J Wainwright.
\newblock Simple, robust and optimal ranking from pairwise comparisons.
\newblock {\em Journal of machine learning research}, 18(199):1--199, 2017.

\bibitem{shalev2014understanding}
Shai Shalev-Shwartz and Shai Ben-David.
\newblock {\em Understanding machine learning: From theory to algorithms}.
\newblock Cambridge university press, 2014.

\bibitem{CS6783}
Karthik Sridharan.
\newblock Cornell cs6783 (machine learning theory), lecture notes: Rademacher
  complexity, 2014.
\newblock URL: \url{http://www.cs.cornell.edu/courses/cs6783/2014fa/lec7.pdf}.

\bibitem{train2009discrete}
Kenneth~E Train.
\newblock {\em Discrete choice methods with simulation}.
\newblock Cambridge university press, 2009.

\bibitem{tversky1969intransitivity}
Amos Tversky.
\newblock Intransitivity of preferences.
\newblock {\em Psychological review}, 76(1):31, 1969.

\bibitem{tversky1972elimination}
Amos Tversky.
\newblock Elimination by aspects: A theory of choice.
\newblock {\em Psychological review}, 79(4):281, 1972.

\bibitem{tversky1991loss}
Amos Tversky and Daniel Kahneman.
\newblock Loss aversion in riskless choice: A reference-dependent model.
\newblock {\em The quarterly journal of economics}, 106(4):1039--1061, 1991.

\bibitem{tversky1992advances}
Amos Tversky and Daniel Kahneman.
\newblock Advances in prospect theory: Cumulative representation of
  uncertainty.
\newblock {\em Journal of Risk and uncertainty}, 5(4):297--323, 1992.

\bibitem{tversky1993context}
Amos Tversky and Itamar Simonson.
\newblock Context-dependent preferences.
\newblock {\em Management science}, 39(10):1179--1189, 1993.

\bibitem{adarank}
Jun Xu and Hang Li.
\newblock Adarank: a boosting algorithm for information retrieval.
\newblock In {\em Proceedings of the 30th annual international ACM SIGIR
  conference on Research and development in information retrieval}, pages
  391--398. ACM, 2007.

\bibitem{zaheer2017deep}
Manzil Zaheer, Satwik Kottur, Siamak Ravanbakhsh, Barnabas Poczos, Ruslan~R
  Salakhutdinov, and Alexander~J Smola.
\newblock Deep sets.
\newblock In {\em Advances in Neural Information Processing Systems}, pages
  3391--3401, 2017.

\end{thebibliography}


\clearpage

\begin{appendices}

\section{Inductive Bias} \label{app:ib}

As stated in Claim 1, SDA with inductive bias Eq. (\ref{eq:agg_ib}) generalizes many models that have been proposed in the discrete choice literature. Below, we summarize the specific choices of $w$, $r$ and $\mu$ in each paper.

\begin{table*}[ht!]
	\centering
	\scalebox{1}{
		\begin{tabular}{cllll}
			Type & Extends & $w$ & $r$ & $\mu$ \\
			\toprule
			- & MNL (\cite{mcfadden1973conditional}) & one & zero & identity \\
			\midrule
			\multirow{4}[2]{*}{SDW} & \cite{tversky1969intransitivity} & (max-min)$^\rho$ & zero & identity \\
			& \cite{mcfadden1978modeling} & linear & log $\sum$ exp & log \\
			& \cite{kalai2002rationalizing} & softmax & zero & identity \\
			& \cite{orhun2009optimal} & linear & w. average & kinked lin. \\
			\midrule
			\multirow{3}[2]{*}{SDE} & \cite{kaneko1979nash} & sum & min & log \\
			& \cite{kivetz2004alternative} (LAM) & sum & (max+min)/2 & kinked lin. \\
			& \cite{kivetz2004alternative} (CCM) & sum & min & power$(\rho)$ \\
			\midrule
			\multirow{1}[2]{*}{SDA} & \cite{kivetz2004alternative} (NCCM) & max-min & min & norm. pow$(\rho)$ \\
			& SDA (ours) & set-nn & set-nn & kinked tanh \\
			\bottomrule
		\end{tabular}%
	}%
	\label{tbl:gen_agg}
	\caption{Discrete choice models as set-aggregation models.}
\end{table*}

As discussed in Sec. \ref{sec:experiments}, we chose our specification of SDA based on set neural networks and a 'kinked' tanh function. In Appendix \ref{app:spec}, we describe each part of our model in further detail.

\section{Approximation Error} \label{app:approx}


We begin with some useful definitions and lemmas due to \cite{ambrus2015rationalising},
rephrased to align with our setup.
We then briefly highlight the differences in setting and results between our paper and
\cite{ambrus2015rationalising}.
Finally, we give the proof of Theorem \ref{thm:err_decomp}.

\subsection{Definitions and lemmas from \cite{ambrus2015rationalising}}

As in many works in economics, \cite{ambrus2015rationalising}
use the concept of item \emph{utility functions} $u$,
mapping each item to a scalar representing its utility.
From our point of view, a score function $f$ correspond to a utility function
if it is \emph{fully parameterize}, i.e., has one parameter for every item
that represents its utility.
We will think of $u$ either as functions $u(x)$ or vectors $u_x$
in appropriate context.
We use $\u$ to define a collection of utilities $u_i$ in the same way that
$\f$ represents a collections of score functions $f_i$,
and accordingly define $g_\u$ as an aggregator of utilities $\u$.

It will be useful to notationally differentiate the \emph{aggregation mechanism} $\mech$,
which is the functional form defining how individual score functions are combined,
from aggregator classes $\Agg$, which include the actual aggregators (i.e., functions),
whose form is given by $\mech$.

\begin{definition}[Triple basis, \cite{ambrus2015rationalising}]
	Let $\mech$ be an aggregation mechanism and $k \in \N$,
	then $\u = \{u_i\}_{i=1}^k$ with $u_i \in \R^3$ is a \emph{\textbf{triple basis}}
	for $T=(x_1,x_2,x_3)$ under $\mech$ if:
	\begin{enumerate}
		\item 
		$g_u(x_1,\{x_1,x_2\}) > g_u(x_2,\{x_1,x_2\})$, and
		\item
		$g_u(x,s') = g_u(x',s') \,\, \forall s' \neq \{x_1,x_2\}, x,x' \in s'$.
	\end{enumerate}
	\label{def:tb}
where $g_\u$ denotes an aggregation with item values given by $\u$.
\end{definition}
When used in an aggregator, triple bases serve two purposes:
they determine the choice from $\{x_1,x_2\}$, but make sure 
this does not effect choices in other choice sets.
Triple bases can be extended to handle \emph{set-triples} $(x,s,t)$, $x \in \X$, $s,t \subseteq \X$
with $x \in s$,
where $x$ plays the role of $x_1$, subsets $s' \subseteq s$ the role of $x_2$,
and subsets $t' \subseteq t$ the role of $x_3$:
the triple bases will choose $x$ from $s'$ (when $x \in s'$) and will be indifferent otherwise.
The following excerpt from the main proof of \cite{ambrus2015rationalising}
(stated here as a lemma) formalizes this notion.
As we now consider multiple items, it will be useful to think of utilities $u$
as mappings from items to scalar utilities, $u:\X \rightarrow \R$,\footnote{While $u$ is a mapping of items, for notational clarity we write it as a function of (observed) features.}
and we will use both representations of utilities (as vectors and as functions) interchangeably.
\begin{lemma}[Triple basis for sets, \cite{ambrus2015rationalising}]
	Let $\u = \{u_i\}_{i=1}^k$ with utilities $u_i:\X\rightarrow \R$,
	and consider some set-triple $(x,s,t)$ with $x \in \X$, $s,t \subseteq \X$
	and $x \in s$.
	If each $u_i(z)$ is the same for all $z \in s \setminus \{x\}$
	and is also the same for all $z \in t$,
	then if $\u$ is a triple-basis for some $(x_1,x_2,x_3)$,
	it is also a \emph{set} triple basis for $(x,s,t)$, in the sense that:
	\begin{enumerate}
		\item $x$ is chosen from any $s' \subseteq s \cup \{x\}$ with $x \in s'$, and
		\item $\u$ is indifferent on all other choice sets, i.e.,
		those without $x$ or that include items from $t$.
	\end{enumerate}
	\label{lemma:tb_sets}
\end{lemma}

Triple bases serve as the main building block in \cite{ambrus2015rationalising}.
The surprising finding in \cite{ambrus2015rationalising} is that once a
triple basis is known for an arbitrary triplet $(x_1,x_2,x_3)$,
using Lemma \ref{lemma:tb_sets}, it can be applied to any violating set $s$
by setting $x=y$ and $t=\X \setminus s$.
Hence, the existence of a triple basis is a property of \emph{the aggregation mechanism},
and not specific to certain items or choice sets (and in our case, to score function classes or the distribution).
The authors of \cite{ambrus2015rationalising} provide explicit constructions of triple bases
for several aggregation (a.k.a. `multi-self') models from the literature,
as well as a general recipe that applies to a large class of
aggregation mechanisms that satisfy, in addition to certain natural axioms
(e.g., Neutrality and Consistency),
the property of \emph{scale invariance}:
\begin{definition}[Scale invariance, \cite{ambrus2015rationalising}]
	An aggregator $g_u$ is scale invariant if there exists an invertible and odd function $\Phi$
	such that for every $\alpha>0$, $g_{\alpha u}(x,s) = \Phi(\alpha) g_u(x,s)$ for all $x \in \X, s \in \allsets$.
	\label{def:scale_inv}
\end{definition}
Hence, for any aggregation mechanism that satisfies Def. \ref{def:scale_inv}
there exists some triple basis $\u$.
Furthermore, the authors show that $\u$ includes at most $k=5$ utilities.
Scale invariance simply means that the ranking over items (and hence prediction)
induced by the aggregator do not depend on the scale of their internal score functions,
and hence, scaling these does not change predictions.
Accordingly, and following \cite{ambrus2015rationalising},
our results herein apply to such mechanisms,
for which we assume the existence of a corresponding triple basis.

The authors of \cite{ambrus2015rationalising} show that
all of their results follow through when exact triple bases are replaced
with approximate triple basis, where the equalities hold only up to some
precision $\epsilon$.
This will also be the case in our proof.

\subsection{Comparing our setting and results to those of \cite{ambrus2015rationalising}}

Before proceeding with our proof, we describe the differences between our result
and those of \cite{ambrus2015rationalising},
thus highlighting some of the challenges encountered while proving our result.

The setting of \cite{ambrus2015rationalising} differs from ours in three crucial aspects.
First, they focus on a realizable setting:
they assume the existence of a choice function $c: \allsets \rightarrow \X$
designating choices $y = c(s) \in s$ for all $s \in \allsets$,
and aim to recover it from the class of all choice functions.
We, on the other hand, focus on the agnostic setting,
where labels (i.e., choices) $y$ are not necessarily generated from a function
within the class we consider (and in fact, can be sampled from an unknown
conditional distribution $D_{Y|S}$).

Second, \cite{ambrus2015rationalising} provide worst case results for the
successful reconstruction of $c$ from any collection of labeled examples.
In contrast, we focus on a statistical setting where a sample set
of alternatives and choices are drawn from some unknown distribution. We are interested in minimizing the expected loss, i.e.,
the probability of correctly predicting choice from choice sets drawn
from the same distribution.

Finally, and perhaps most importantly, \cite{ambrus2015rationalising}
assume that all items can be given arbitrary scores (which they refer to as `utilities').
In other words, score functions are fully parameterized and can assign any value to
any item.
This of course means that the number of parameters required \emph{for each score function}
(of which there can be many in the context of aggregation)
is equal to the number of items, which in practice can be rather large,
and in principle can be unbounded.
In contrast, we focus on the parametric settings common in machine learning,
where items are described by features (e.g., vectors),
and score functions are parametric functions of those features.
Critically, without the assumption of full parameterization,
the results of \cite{ambrus2015rationalising} break down,
as they require the ability to assign arbitrary values to each item.

\subsection{Proof of Theorem \ref{thm:err_decomp}}

Let $\mech$ be a scale-invariant aggregation mechanism
as in \cite{ambrus2015rationalising},
and let $\u$ be a corresponding triple basis of size $k$,
which we assume fixed throughout the proof.
As noted in Sec. \ref{sec:apx_err}, the
proof requires that the class of aggregators $\Agg$ be defined over
a class of score functions that is slightly more expressive than $\F$,
denoted $\Fplus$.
Specifically, $\Fplus$ includes combinations of pairs of score functions from $\F$,
given by $\Fplus = \{a(b(x), b'(x)) \,:\, a \in \nn, b,b' \in \Base \}$
where $\nn$ is a class of small neural networks
(2 inputs, 2 layers with 2 units each, sigmoid activations) 
whose precise definition will be given in the proof.
Note that by construction we will have $\F \subset \Fplus$.

To reduce notational clutter we will use $\err{*}{\dist}(\cdot) = \err{*}{}(\cdot)$ for the minimal expected error
and $\err{*}{A}(\cdot)=\err{*}{}(\cdot|A)$ for the minimal expected error
conditioned on the event $s \in A$.

We can now restate a slightly tighter variant of Theorem \ref{thm:err_decomp} in finer detail:
\begingroup
\def\thetheorem{\ref{thm:err_decomp}}
\begin{theorem} 
Let $\mech$ be a scale-invariant aggregation mechanism as in \cite{ambrus2015rationalising}
with a triple basis $\u$ of size $k$.
Let $\Agg = \Agg^{(\ell)}_\Fplus$ be a corresponding class of aggregators 
of dimension $\ell$ defined over the class of item score functions $\Fplus$. Then:
\beq
\err{*}{\dist}(\Agg) \le
p_\Nonviol  \, \err{*}{\Nonviol}(\Fplus) + 
\min_{B \in \Base^{\ell'}}
\sum_{b \in B} p_{\sep_b}  \, \err{*}{\sep_b}(\Base) + p_{\Viol \setminus \sep_B}
\label{eq:err_decomp_appendix}
\eeq
where 
$\sep_B = \cup_{b \in B} \sep_b$,
$\ell' \le (\ell-1)/k$,
and $p_A=P(s \in A)$.
\end{theorem}
\addtocounter{theorem}{-1}
\endgroup

Before giving the proof, we state a practical corollary of Theorem \ref{thm:err_decomp}.
\begin{corollary}
Let $\nn^{N,K}$ be a class of neural networks with $N \ge 2$ fully connected layers
with $K \ge 2$ units each and with sigmoidal activations (i.e., multilayer perceptrons).
Then if $\Agg$ is defined over score functions $\nn^{N,K}$,
the bound in Eq. \ref{eq:err_decomp_appendix} holds for 
$\F = \nn^{N-2,K/2}$.
\end{corollary}
The corollary shows how for fully-connected neural networks,
the error on each violating region is bounded by the error of
slightly less expressive neural networks.
The result holds since $\Fplus \subseteq \nn^{N,K}$.

\begin{proof}

We begin with two useful definitions.
\begin{definition}[Implementation]
	Let $\f \in (\Fplus)^k$, then if $\f(x')=\u(x')$ for all items $x'$
	appearing in a (set-)triple $T$,
	we say that $\f$ \emph{\textbf{implements}} $T$.
	Similarly, if $\f(x') \approx \u(x')$ for all $x'$,
	we say that $\f$ approximately-implements $T$.
	\label{def:implement}
\end{definition}

\begin{definition}[Isolation]
	Let $s \in \allsets$,
	then if $\f$ (approximately) implements 
	the set-triple $(z,s,\coap(s))$ 
	for some $z \in s$ and for $\coap(s)=\allsets \setminus s$,
	we say that $\f$ (approximately) \emph{\textbf{isolates}} $s$.
	\label{def:isolate}
\end{definition}

Implementation simply states that the values of items in the triple-basis
under $\f$ are (approximately) those under $\u$.
Isolation considers the implementation of triple bases for which the choice set is $s$
(but the actual choice $z$ does not matter).
As we will see, aggregation will target certain choice sets $s$ by ``isolating''
them from others, letting some of the aggregated score functions effect $s$,
but guaranteeing that others are indifferent to it.
Note that implementation and isolation, as well as separation (Def. \ref{def:separation}),
are inherited properties, i.e., that if they hold for $s$,
they also hold for any $s' \subseteq s$.

For our next lemma, we will make explicit the class $\nn$.
Each $a \in \nn$ is a neural networks taking as input vectors of size two and outputing a single scalar.
The networks have two fully-connected hidden layers, each with two units, and sigmaoidal activations.
The final layer is a 2-to-1 linear layer.

We parameterize units using $r(\cdotp;\alpha,\beta)=\inner{\alpha,\cdotp}+\beta$
with $\alpha \in \R^2$, $\beta \in \R$.
and use $\theta \in \Theta$ to denote all of the network's parameters.
A network $a \in \nn$ with parameters $\theta$ is denoted $a_\theta$.
We assume w.l.o.g. that sigmoidal activations $\sig$ are scaled to $[0,1]$.


Recall that each function $\fbar \in \Fplus$ is composed of a pair $b,b' \in \Base$
whose outputs are combined via some $a \in \nn$,
given by $\fbar_\theta(x;b,b') = a_\theta(b(x),b'(x))$.
We further denote:
\[
\f_\btheta(x;b,b') = (f_{\theta_1}(x;b,b'),\dots,f_{\theta_k}(x;b,b')), \quad
\btheta = (\theta_1,\dots,\theta_k)
\]
The next lemma shows how functions in $\Fplus$ can implement triple bases when separation holds.

\begin{figure}[!t]
	\begin{center}
		\includegraphics[width=0.8\columnwidth,trim=0 0 4cm 0,clip]{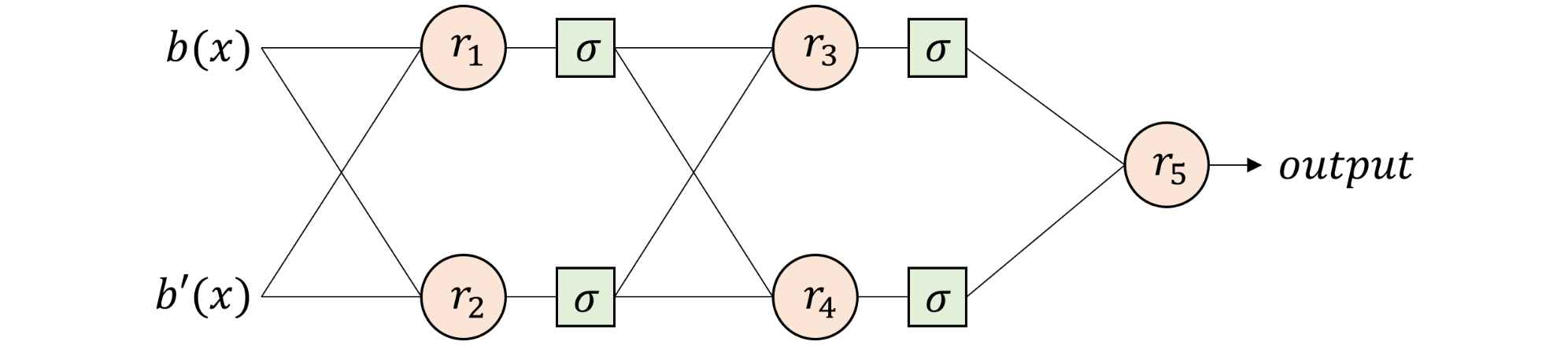} 
		\caption{The auxiliary class $\nn$}
		\label{fig:nn_viz}
	\end{center}
\end{figure}

\begin{lemma}[Neural implementation of triple bases]
	Let $T=(x_1,x_2,x_3)$, and let $b,b' \in \Base$ be such that $b$ separates $x_1,x_2$ from $x_3$
	and $b'$ separates $x_1$ from $x_2$, i.e.,
	$\{x_1,x_2\} \succarg{b}	x_3$ and $x_1 \succarg{b'} x_2$.\\
	Then, there exists $\btheta \in \Theta^k$ for which $\f_\btheta(x;b,b')$
	(approximately) implements $T$.
	\label{lemma:nn_imp_tb}
\end{lemma}

\begin{proof}
	We give a general recipe for constructing $\btheta$.
	Assume w.l.o.g. that $\sig$ maps to $[0,1]$.
	The idea is to choose $\alpha,\beta$ such that $\sig$
	will give items on the r.h.s. and l.h.s. of the separation operator $\succ$ values that are arbitrarily close to 0 and 1, respectively.
	We now construct each unit of the the neural network (see diagram in Fig. \ref{fig:nn_viz}).
	\begin{itemize}
		\item 
		Since $\{x_1,x_2\} \succarg{b}	x_3$,
		there exist $\alpha,\beta$ such that
		$r((b(x),b'(x));\alpha,\beta) \approx 1$ for $x=x_1,x_2$ and 0 for $x=x_3$.
		This is because $\sig$ is a sigmoidal, and hence $\alpha$ and $\beta$ simply shift
		and scale the sigmoid so that the higher-valued $x_1,x_2$ are ``pushed'' towards 1
		and the lower-valued $x_3$ towards 0.
		We denote this unit by $r_1$.
		\item 
		Since $x_1 \succarg{b'}	x_2$, there exist $\alpha,\beta$ such that
		$r((b(x),b'(x));\alpha,\beta) \approx 1$
		for $x=x_1$ and 0 for $x=x_2$. We denote this unit by $r_2$.
		Note that the values for $x=x_3$ are arbitrary but bounded in $[0,1]$.
		\item 
		There exist $\alpha,\beta$ such that
		$r((r_1(x),r_2(x));\alpha,\beta) \approx 1$ for $x=x_1$
		and 0 for $x=x_2,x_3$.
		This is because $r_1$ and $r_2$ contribute 1 to $x_1$,
		while $x_2$ and $x_3$ never get 1 from $r_1$ and $r_2$.
		We denote this unit by $r_3$.
		\item 
		Because $\sig$ is sigmoidal,
		there exist $\alpha=(\epsilon,0)$ with small enough $\epsilon$ such that 
		with $\beta=0$, $r$ approximates the identity function on the first input.\footnote{
			Alternatively, $\nn$ can be defined with only one unit in the second layer.}
		We denote this unit by $r_4$.
	\end{itemize}
	
	Note that measuring the outputs of $r_3$ and $r_4$ when plugging $x_1,x_2,x_3$ into $b,b'$
	gives:
	\begin{equation} \label{eq:linear_basis}
	r_3 \begin{pmatrix} x_1 \\ x_2 \\ x_3 \end{pmatrix}  \approx
	\begin{pmatrix} 1 \\ 0 \\ 0 \end{pmatrix}, \qquad
	r_4 \begin{pmatrix} x_1 \\ x_2 \\ x_3 \end{pmatrix}  \approx 
	\begin{pmatrix} 1 \\ 1 \\ 0 \end{pmatrix} 
	\end{equation}
	which together with the vector $e=(1,1,1)$, form an (approximate) linear basis for $\R^3$.
	This means that any utility $u_i \in \u$ when applied to $x_1,x_2,x_3$
	can be expressed as a linear combination of $r_3,r_4$, and $e$.
	Specifically,	there exist $\alpha, \beta$ such that the linear layer ($r_5$) gives
	$\inner{\alpha,(r_3(x),r_4(x))}+\beta = u_i(x)$ for all $x \in \{x_1,x_2,x_3\}$.
	Altogether, we get that there exist $\theta_i$ with $f_{\theta_i}(x;b,b')=u_i$ for all $i$,
	giving $\btheta = \{\theta_i\}_{i=1}^k$ as required.
\end{proof}

\begin{corollary} \label{corr:nn_imp_tb_sets}
Lemma \ref{lemma:nn_imp_tb} applies to set-triples $(x,s,t)$
when $s \succarg{b} t$ and $x \succarg{b'} s$.
\end{corollary}
\begin{proof}
According to Lemma \ref{lemma:tb_sets}, it suffices to ensure that
scores for all $z \in s \setminus \{x\}$ and for all $z' \in t$ are the same.
Since the construction in Lemma \ref{lemma:nn_imp_tb} applies to arbitrary separable triples,
the above can be achieved by choosing parameters (i.e., scaling and shifting)
such that the properties of each $r_i$ hold for all $x \in s,t$.
\end{proof}

Lemma \ref{lemma:nn_imp_tb} 
give a ``template'' for generating set-triple bases from a pair
of score functions $b,b'$, where $b$ determines the choice set $s$ (by separating $s$ from $t$),
and $b'$ determines the choice $x \in s$ (by separating $x$ from $s$).
This notion is formalized in the next result relating separation and isolation.
\begin{lemma}[Separation entails isolation]
	If $b$ separates $s$, then for any $b' \in \Base$,
	we have that $\f_\btheta(x;b,b')$ isolates $s$.
	Furthermore, the prediction is determined by $b'$, i.e.,
	\[
	h_{g_{_\f}}(s) = \argmax_{x \in s} b'(x)
	\]
	\label{lemma:sep2iso}
\end{lemma}

\begin{proof}
	Let $z = \argmax_{x \in s} b'(x)$, then instantiate Lemma \ref{lemma:nn_imp_tb}
	for $b,b'$ on $T=(z,s,\coap(s))$.
\end{proof}

Note that $\btheta$ is determined by the aggregation mechanism $\mech$
(this is because $\u$ is determined by $\mech$).
Hence, once $\btheta$ is fixed, the triple-basis template is fixed,
and learning can focus on targeting choice sets
(by optimizing $b$) and predicting within those choice sets (by optimizing $b'$).
In the remainder of the we will use $\btheta$ to denote the template parameters
corresponding to $\u$ (note that Lemma \ref{lemma:nn_imp_tb} implies
that such $\btheta$ exists for any triple basis).
Because the approximation in Eq. \eqref{eq:linear_basis} can be made arbitrarily small
(i.e., by scaling the appropriate $\alpha$-s),
the approximate-indifference of the implemented triple-basis can also be arbitrarily small.
Thus, $\btheta$ can be made to give $\epsilon$-approximate indifference
for any necessary $\epsilon$.

We can now define the following class:
\[
\Fplussup{b} = \{ \f_\btheta(x;b,b') \,:\,  b' \in \Base \}
\]
which includes all score functions that implement set-triple bases
that effect choice sets separated by $b$
(while predictions can vary across functions according to $b'$).
\begin{corollary} \label{corr:sep_iso_all_F}
	If $b$ separates $s$, then $s$ is isolated by all $\f \in \Fplussup{b}$.
\end{corollary}
Isolation will be our main building block for showing error decomposition.
For $b \in \Base$, we denote:
\[
\iso_b = \{s \in \Viol \,:\, s \text{ is isolated by all } \f \in \Fplussup{b}\}
\]
Corollary \ref{corr:sep_iso_all_F} implies that $\sep_b \subseteq \iso_b$.

Note that although each $\f_\btheta(x;b,b')$ includes only two base score functions $b,b' \in \F$,
it is in fact composed of $k$ score functions $\f_1,\dots,\f_k \in \Fplus$,
each determined by one of $\theta_1, \dots \theta_k \in \btheta$,
or in other words, $\Fplussup{b} \subset (\Fplus)^k$.
We denote by $\Agg_{\Fplussup{b}}$ the class of $k$-dimensional aggregators over $\Fplussup{b}$.

The following lemma shows how isolation helps in decomposing
the error of $\Agg$.
By ``isolating'' a separable region of $\allsets$,
a budget of $k$ (out of $\ell$) score functions can be allocated to that region,
and aggregation will ensure that these score functions will only effect predictions
of choice sets within the region.
This is the main component in the decomposition bound of the error of $\Agg$.
\begin{lemma}[Decomposition by isolation]
For any $R \subseteq \allsets$ with $\iso_b \subseteq R$ and any $n \ge k$,
	it holds that:
	\begin{equation}
	\err{*}{R}(\Agg^{(n)}_\Fplus) \le
	p_{\iso_b} \err{*}{\iso_b}(\Base) +
	p_{R \setminus \iso_b} \err{*}{R \setminus \iso_b}(\Agg^{(n-k)}_\Fplus)
	\end{equation}
	\label{lemma:decomp_via_iso}
\end{lemma}

\begin{proof}
	Let $\Agg' = \Agg_{\Fplussup{b}} \times \Agg^{(n-k)}_\Fplus$, i.e.,
	aggregators whose first $k$ score functions give some $\f \in \Fplussup{b}$.
	Since $\Agg' \subseteq \Agg$, the minimal error cannot decrease: 
	\beq
	\err{*}{R}(\Agg^{(n)}_\Fplus) \le \err{*}{R}(\Agg') 
	\label{eq:decomp_gprime}
	\eeq
	Next, fix some constant $c$, and denote by $\bar{\Agg} \subset \Agg^{(n-k)}_\Fplus$
	the aggregators whose absolute values are at most $c$,
	and by $\ubar{\Agg} \subseteq \Agg_{\Fplussup{b}}$ the aggregators whose absolute values
	are at least $c$.\footnote{Here we make the technical assumption that $\agg(x,s) \neq 0$ for all $x,s$.}
	Denoting $\Agg'' = \ubar{\Agg} \times \bar{\Agg}$,
	because $\bar{\Agg} \subseteq \Agg^{(n-k)}_\Fplus$ and
		$\ubar{\Agg} \subseteq \Agg_{\Fplussup{b}}$, we have:
	\beq
	\err{*}{R}(\Agg') \le \err{*}{R}(\Agg'') 
	\label{eq:decomp_gprimeprime}
	\eeq

	For any $\agg \in \Agg''$, since $\agg$ is an aggregator, 
	we can write $\agg = \ubar{\agg} + \bar{\agg}$
	where $\ubar{\agg} \in \ubar{\Agg}$ and $\bar{\agg} \in \bar{\Agg}$.
	We now consider how different $s \in R$ are treated by $\ubar{\agg}, \bar{\agg}$, and $\agg$.
	If $s \in \iso_b$, then because $\ubar{\agg}(x,s) > \bar{\agg}(x',s)$ for all $x,x' \in s$,
	$\ubar{\agg}$ `dominates' $\bar{\agg}$ in the sense that only $\ubar{\agg}$ effects predictions,
	i.e., $h_\agg(s) = h_{\ubar{\agg}}$.
	On the other hand, if $s \in R \setminus \iso_b$,
	then because $\ubar{\agg} \subset \Agg_{\Fplussup{b}}$,
	it is approximately-indifferent on $s$
	(to see this, note that due to the inheritance of isolation,
	there is no $t \supset s$ with $t \in \iso_b$,
	and so $s$ falls in the second class of sets from Lemma \ref{lemma:tb_sets}).
	And since this approximate-indifference can be made arbitrarily small
	(by adjusting $\btheta$ as described above),
	$\ubar{\agg}$ does not effect predictions, i.e., $h_\agg(s) = h_{\bar{\agg}}$.
	Overall, for any $\agg \in \Agg''$, we have:
	\[
	h_\agg(s) = 
	h_{\ubar{\agg}}(s)\one{s \in \iso_b} + 
	h_{\bar{\agg}}(s)\one{s \in R \setminus \iso_b}
	\]
	This allows us to decompose the error:
	\begin{equation} \label{eq:decomp_primeprime}
	\err{*}{R}(\Agg'') \le
	p_{\iso_b} \err{*}{\iso_b}(\ubar{\Agg}) +
	p_{R \setminus \iso_b} \err{*}{R \setminus \iso_b}(\bar{\Agg})
	\end{equation}
	and because $\mech$ is scale-invariant (Definition \ref{def:scale_inv}),
	scaling does not change predictions
	(and thus does not change the expressivity of the classes),
	allowing us to remove the constraints on $c$:
		\begin{equation} \label{eq:decomp_noc}
	\err{*}{\iso_b}(\ubar{\Agg}) = \err{*}{\iso_b}(\Agg_{\Fplussup{b}}) , \qquad
	\err{*}{R \setminus \iso_b}(\bar{\Agg}) = \err{*}{R \setminus \iso_b}( \Agg^{(n-k)}_\Fplus)
	\end{equation}

	Finally, because the predictions of any $\f(x;b,b') \in \F_\btheta(b)$ are determined by $b'$
	(Lemma \ref{lemma:sep2iso}),
	we can replace $\Agg_{\F_\btheta(b)}$ in Eq. \eqref{eq:decomp_noc} with $\Base$,
	and combining Eqs. \eqref{eq:decomp_gprime}-\eqref{eq:decomp_noc}
	concludes the proof.
	
\end{proof}

We now tie all lemmas together to derive Eq. \eqref{eq:err_decomp_appendix}.
Recall that $\ell = k \ell' + 1$.
For any $B=\{b_1,\dots,b_{\ell'}\}$
with $\cap_{b \in B}\iso_b = \varnothing$,
we can iteratively applying Lemma \ref{lemma:decomp_via_iso}
to isolate separable regions until our budget is depleted:
\begin{align*}
\err{*}{D}(\Agg^{(\ell)}_\Fplus) & \le
p_{\iso_{b_1}} \err{*}{\iso_{b_1}}(\Base) +
p_{\allsets \setminus \iso_{b_1}} \err{*}{\allsets \setminus \iso_{b_1}}(\Agg^{(\ell-k)}_\Fplus) \\
& \le
p_{\iso_{b_1}} \err{*}{\iso_{b_1}}(\Base) +
p_{\iso_{b_2}} \err{*}{\iso_{b_2}}(\Base) +
p_{\allsets \setminus \iso_{b_1} \cup \iso_{b_2}} 
\err{*}{\allsets \setminus \iso_{b_1} \cup \iso_{b_2}}(\Agg^{(\ell-2k)}_\Fplus) \\
& \le \dots \\ &\le
\sum_{b \in B} p_{\iso_b} \err{*}{\iso_b}(\Base) +
p_{\allsets \setminus \iso_B} \err{*}{\allsets \setminus \iso_B}(\Agg^{(1)}_\Fplus) \\ 
&\le
\sum_{b \in B} p_{\iso_b} \err{*}{\iso_b}(\Base) +
p_{\allsets \setminus \iso_B} \err{*}{\allsets \setminus \iso_B}(\Fplus)
\end{align*}
where $\iso_B = \cup_{b \in B} \iso_b$,
and the last inequality is due to $\Fplus \subseteq \Agg^{(1)}_\Fplus$
(note that also $\F \subseteq \Fplus$, giving the version of the Theorem in the paper).

Because the above holds for any $B$, we can plug in the optimal $B$.
We now have:
\begin{align*}
\err{*}{D}(\Agg^{(\ell)}_\Fplus) & \le
\min_{B \in \Base^{\ell'}}
\sum_{b \in B} p_{\iso_b} \err{*}{\iso_b}(\Base) +
p_{\allsets \setminus \iso_B} \err{*}{\allsets \setminus \iso_B}(\Fplus) \\
& \le
\min_{B \in \Base^{\ell'}}
\sum_{b \in B} p_{\iso_b} \err{*}{\iso_b}(\Base) +
p_{\Viol\setminus \iso_B} + 
p_{\Nonviol} \err{*}{\Nonviol}(\Fplus)\\
& \le 
\min_{B \in \Base^{\ell'}}
\sum_{b \in B} p_{\sep_b} \err{*}{\sep_b}(\Base) +
p_{\Viol\setminus \sep_B} + 
p_{\Nonviol} \err{*}{\Nonviol}(\Fplus)
\end{align*}
The first inequality holds because isolation is an inherited property,
and hence the argmin $B^*$ necessarily has disjoint isolated regions $\{\iso_b\}_{b\in B^*}$.
The second inequality holds because $\iso_B$ only includes sets in $\Viol$,
and $\err{*}{\Viol \setminus \iso_B}(\cdot) \le 1$.
The third inequality holds since by Lemma \ref{lemma:sep2iso}, separation entails isolations,
and so $\sep_b \subseteq \iso_b$, meaning that sets $s \in \iso_b \setminus \sep_b$
are assumed to receive the worst-case error of one
(i.e., are ``moved'' from the error term $\err{*}{\iso_b}$
to the maximal-error term $p_{\Viol \setminus \sep_b}$),
concluding the proof.
\end{proof}


\section{Estimation Error} \label{app:est}


We begin with the case $g(x,s) = \inner{w(s),\f(x)}$ with:
\[
\f(x) = (f_1(x),\dots,f_\ell(x)), \qquad
w(s) = (w_1(s), \dots, w_\ell(s))
\]
where $w_i(s) = w(f_i(s))$,
$f_i(s) = (f_i(x_1),\dots,f_i(x_n))$,
and each $f_i(x) = \inner{\theta_i,x}$.

We can write $\f(x) = x^\top \Theta$
where $\Theta_{i \cdotp} = \theta_i$ are rows,
and denote columns by $\bar{\theta}_j = \Theta_{\cdotp j}$.
Note that:
\begin{align*}
g(x,s) = \sum_{i=1}^\ell w_i(s) \inner{\theta_i,x} 
= \sum_{k=1}^d \sum_{i=1}^\ell \Theta_{ik} w_i(s) x_k 
= \sum_{k=1}^d \underbrace{\inner{\bar{\theta}_k, 
		\overbrace{w(s) \cdotp x_k}^{\text{(I)}}}}_{\text{(II)}}
\end{align*}

We now bound the Rademacher complexity of each component.
\begin{align*}
R_f & \le X_\infty \sqrt{\frac{2 \log 2d}{m}}  &
\text{(\cite{shalev2014understanding}, Lemma 26.11)} \\
R_w & \le \lambda_w^{(\rho)} R_f &
\text{(\cite{shalev2014understanding}, Lemma 26.9)}\\
R_{\text{(I)}} & \le \max_x \|x\|_\infty R_w &
\text{(\cite{shalev2014understanding}, Lemma 26.6)}\\
R_{\text{(II)}} & \le 2 \max_{\bar{\theta}} \| \bar{\theta} \|_1 \cdotp R_{\text{(I)}} 
= 2 \|\Theta\|_1 R_{\text{(I)}} &
\text{(\cite{CS6783}, Sec. 4)}\\
\end{align*}
Assuming $\|\Theta\|_1 \le 1$, combining the above gives:
\[
R_g \le 2 X_\infty^2\lambda_w^\rho\sqrt{\frac{2 \log 2d}{m}}
\]
Using standard Rademacher-based generalization bounds (i.e., \cite{bartlett2002rademacher})
concludes the proof:
\[
\err{}{}(\G) \le \err{}{\smplset}(\G) + 2R_g + O\left( \sqrt{\frac{\log(1\delta)}{m}}\right)
\]
gives the desired result.

We now analyze the case $g(x,s) = \inner{v, \nonlin(\f(x)-r(s))}$, where
$v \in \R^\ell$, 
$r(s) = (r_1(s), \dots, r_\ell(s))$,
and $r_i(s) = r(f_i(s))$.
The Rademacher complexity of each component is:
\begin{align*}
R_\nonlin & \le \lambda_\nonlin (R_f + R_r)  &
\text{(\cite{shalev2014understanding}, Lemma 26.6)} \\
R_r & \le \lambda_r^{(\rho)} &
\text{(\cite{shalev2014understanding}, Lemma 26.6)} \\
R_g & \le 2 W_1 R_\mu &
\text{(\cite{shalev2014understanding}, Lemma 26.7)}
\end{align*}
and $R_f$ is as before. Together, this gives:
\[
R_g  \le 2 W_1 R_\mu \le 2 W_1 \lambda_\mu (1+\lambda_r^\rho) X_\infty \sqrt{\frac{2 \log 2d}{m}}
\]
which concludes our proof.


\section{Experiments} \label{app:experiments}

\subsection{Datasets} \label{app:datasets}

For Expedia and Outbrain, the preprocessed data used to train the models can be found here\footnote{https://drive.google.com/file/d/1G0Rmsa9M1D5NKbryY4mTmZCX9i6MuY2F/view}. For Amadeus, the data is not publicly available. Please contact the authors of \cite{mottini2017deep} for the data. Note that we did not conduct any feature engineering on Amadeus dataset. Here are the details of the features and the preprocessing steps of each dataset.

\begin{enumerate}
	\item Amadeus: Each item is a recommended flight itinerary, and user clicks one. Key features include flight origin/destination, price and number of transfers. User features are excluded from the original dataset due to privacy concerns. We did not coduct any feature engineering, and used the dataset from \cite{mottini2017deep}.
	\item Expedia: Each item is a recommended hotel, and user clicks one. Key features include hotel price, rating, length of stay, booking window, user's average past ratings and visitor's location. We applied the following standard preprocessing steps for different variable types:
	\begin{itemize}
		\item continuous: depending on the distribution, log or square root transform to make the distribution look Gaussian. 
		\item ordinal: one-hot encode.
		\item datetime: get week and month to capture seasonality of hotel pricing.
		\item categorical: one-hot encode. For those with too many categories, group unpopular ones as "others".
		\item new features: we created one new feature, popularity score, based on a popular blog post on this dataset\footnote{https://ajourneyintodatascience.quora.com/Learning-to-Rank-Personalize-Expedia-Hotel-Searches-ICDM-2013-Feature-Engineering}. 
	\end{itemize}
	\item Outbrain: Each item is a news article, and user clicks one. When users see an article, they also see these recommended articles at the bottom of the page. Key features include article category, advertiser ID, and geo-location of the views. For preprocessing steps, we followed one of the leading solutions in the Outbrain click preidiction Kaggle competition\footnote{Up to Step 5 of https://github.com/alexeygrigorev/outbrain-click-prediction-kaggle}.
\end{enumerate}

Table \ref{tbl:data} includes further details.
\begin{table}[h!]
	\centering
	\fontsize{8.5pt}{10pt}\selectfont
	\caption{Dataset description}
	\begin{tabular}{lccccc}
		Dataset & $m$ & $|\X|$ & $\max(n)$ & $\avg(n)$ & $d$ \\
		\midrule
		Amadeus & 34K &  1.0M & 50 & 32.1 & 17 \\
		Expedia & 199K & 129K & 38 & 25 & 8 \\
		Outbrain & 16.8M &  478K & 12 & 5.2 & 10 \\
	\end{tabular}%
	\label{tbl:data}%
\end{table}%

\subsection{Baselines} \label{app:baselines}

As shown in Claim \ref{lemma:multiself}, SDA generalizes many known variants of MNL. Our implementation of SDA is capable of running these specific instances of MNL variants. MNL as well as all the models in Figure \ref{fig:extra} (center) are implemented within our framework.

\begin{itemize}
	\item MNL: our implementation.
	\item SVMRank: used an open source code on GitHub\footnote{https://gist.github.com/coreylynch/4150976/1a2983d3a896f4caba33e1a406a5c48962e606c0}, with minor modifications.
	\item RankNet: used learning2rank library open sourced on GitHub\footnote{https://github.com/shiba24/learning2rank}.
	\item MixedMNL: our implementation.
	\item AdaRank: used an open source code on GitHub\footnote{https://github.com/rueycheng/AdaRank}.
	\item DeepSets: used source code provided by the authors\footnote{https://github.com/manzilzaheer/DeepSets}.
\end{itemize}

For neural network based models SDA, RankNet, DeepSets, the number of parameters are $816+784d$, $525312+1024d$, $196864+256d$, respectively, where $d$ is the number of features in a dataset. For a reasonable range of $d$, the number of SDA parameters is significantly lower than that of other models. This further illustrates how SDA reduces model complexity by incorporating inductive bias in clever ways. 

\subsection{Setup} \label{app:setup}
\textbf{Implementation}

All code was implemented in Python, using Tensorflow\footnote{\raggedright\url{https://www.tensorflow.org/}}. The source code can be found here\footnote{\raggedright\url{https://drive.google.com/file/d/1KZVbqfVR6QNIpv38y4e8ptzVdbi_GQH4/view}}.

\textbf{Evaluation Metrics Definition}

\begin{itemize}
	\item top-1 accuracy: the conventional accuracy. $1$ if the model choice prediction is the same as the chosen item, $0$ otherwise. 
	\item top-5 accuracy: $1$ if the model's 5 highest probability choice predictions include the chosen item, $0$ otherwise. 
	\item mean reciprocal rank (MRR): a rank-based measure commonly used in the information retrieval literature. Let $\text{rank}_i$ to indicate the rank of the chosen item in our prediction (by probability). Then, $\text{MRR} = \frac{1}{m} \sum_{i=1}^m \frac{1}{\text{rank}_i}$. Because we want the rank of the model prediction to be higher (and thus the reciprocal rank to be lower), the lower MRR the better.
\end{itemize}

\textbf{Hyperparameters}

For all methods, we tuned regularization, dropout, and learning rate (when applicable) using Bayesian optimization using an open source library Optuna\footnote{\url{https://optuna.org/}}. We tuned the hyperparameters on the validation set for 100 trials of Bayesian optimization. The range of hyper-parameters considered is as follows:

\begin{itemize}
	\item learning rate: log uniformly sample from 1E-05 \textasciitilde 1E-03
	\item weight decay: log uniformly sample from 1E-10 \textasciitilde 1E-03
	\item dropout keep probability: uniformly sample from 0.5 \textasciitilde 1.0
\end{itemize}

For exponential decay, we used decay rate of 0.95 with decay step of 10 for all models. For batch size, we used 128 for all models. Finally, we applied early stopping to all models based on the validation accuracy with an early stop window of 25 epochs.

\textbf{Computing Infrastructure}

We optimized the hyperparameters and trained our model on Slurm Worklord Manager\footnote{\url{https://slurm.schedmd.com/documentation.html}}. All the training was done on CPUs, and the CPU core type we used are AMD "Abu Dhabi" and Intel "Broadwell".


\subsection{Ablation Study} \label{app:ablation}
\textbf{Details on Model Specification} \label{app:spec}

Recall that SDA is of the form:
$$\agg(x,s;w,\feat) = \inner{w(s), \feat(x,s)}$$
We further introduced the general form of inductive bias:
$$w(s) = w(\f(s)), \qquad
\feat(x,s) = \nonlin \big(\f(x) - \r(\f(s)) \big)$$

In this section, we first elaborate on the different components of $\agg$ to supplement Sec. \ref{sec:model}, then justify our choices of $w$, $r$ and $\mu$ described in Sec. \ref{sec:experiments}. 

For a single dimension $i \in [\ell]$ in the embedded space,
items are first evaluated using some $f_i(x) \in \F$.
Then, a ``reference'' valuation is constructed
via a set function $r_i(s)$.
Next, $r_i(s)$ is subtracted from $f_i(x)$ to ``standardize'' the scores with respect to this reference. As noted in \ref{sec:experiments}, we have found it useful to generalize this by letting $f_i$ and $r_i$ be vector valued,
and taking their inner product instead.
The item-reference relation is then fed into an a-symmetric non-linearity $\nonlin(\cdotp)$. Finally, all $\ell$ valuations are aggregated using set-dependent weights $w(s)$.

Eq. \eqref{eq:agg_ib} injects inductive bias 
inspired by the following key principles in behavioral choice theory:
\begin{enumerate}[label=\textbf{P\arabic*}:,ref=P\arabic*]
	\setlength\itemsep{0.2em}
	
	\item \label{p:nonlin}
	\textbf{Asymmetry.} 
	Losses and gains are perceived in an a-symmetric, non-linear manner.
	This is the hallmark of Kahneman and Tversky's Prospect Theory
	\cite{kahneman1979prospect,tversky1992advances}.
	
	\item \label{p:r}
	\textbf{Relativity.}
	Valuations are relative, and are considered with
	respect to a mutual (and possibly hypothetical) referral item
	(e.g., \cite{tversky1991loss}).
	
	\item \label{p:w}
	\textbf{Integrability.}	
	Subjective value is multi-dimensional,
	and context determines how valuations are integrated into a single choice
	(e.g., \cite{tversky1969intransitivity, tversky1993context}).
\end{enumerate}

\begin{proposition}
	$g(x,s)$ in Eq. \eqref{eq:agg_ib} satisfies  \ref{p:nonlin},
	\ref{p:r}, and \ref{p:w}.
\end{proposition}
To see this, note that for each dimension $i$,
item scores $f_i(x)$ are compared to a mutual set-dependent reference point $r_i(s)$
via $d$, satisfying \ref{p:r}.
Using a pointwise a-symmetric non-linearity for $\nonlin$ gives \ref{p:nonlin}.
Finally, in accordance with \ref{p:w}, valuations are aggregated
in $\agg$ via $\w(s)$. 

We now describe our specific choice of components in Sec. \ref{sec:experiments}.
We model $w(s)$ and $r(s)$ as set neural networks \cite{zaheer2017deep}.
These are general neural networks whose architecture guarantees permutation
invariance, meaning that permuting $s$ will not change the output.
For the scalar function $\nonlin$, we have found it useful to use a kinked tanh function:
\[
\nonlin(z) =  c \tanh(z) \cdotp \one{z<0} + \tanh(z) \cdotp \one{z\ge0}, \quad c>1
\]
which is inspired by the s-shaped utility functions used in prospect theory \cite{maggi2004characterization}.

Since all four elements of $\agg$ are differentiable,
$\agg$ can be optimized over any appropriate loss function (i.e., cross entropy)
using standard gradient methods.
The learned parameters include the weights in $w$ and in $r$,
and $c$ for $\nonlin$.
We have noted in Sec. \ref{sec:model} that Eq. \eqref{eq:agg_ib} encompasses several models from the multi-self literature.
Since these are usually designed for mathematical tractability,
they do not always include explicit functional forms,
and use simple non-parametric set operations for $w$ and $r$.
The predictive advantage of our parametrized model is demonstrated empirically
in Sec. \ref{sec:experiments}. 

\textbf{Ablation Study}

We now investigate the contribution of each component of SDA in an ablation study. 

In \ref{sec:model}, we motivated our model choice from the behavioral decision theory perspective. To motivate our design decisions also from a machine learning point of view, we conducted an ablation study. In particular, we decomposed SDA into $F$, $\ell$, $w$, $\varphi$ (which can consist of $r$, $\mu$), removed each component, and analyzed the performance. The full set of ablation models is detailed in Table \ref{tbl:ablation_details} and the experimental results are presented in Table 5. The experiment setup is exactly the same as Sec. \ref{sec:experiments}. 

\begin{table}[h]
	\centering
	\caption{Specification of all ablated models}
	\scalebox{0.9}{
	\begin{tabular}{|l|lrllll|}
		\cline{1-7}      & F & \multicolumn{1}{l}{$\ell$} & $w$ & $\varphi$ & r & $\mu$ \bigstrut\\
		\hline
		SDA & linear & 24 & Set NN & $\langle F(x), r(s) \rangle$ & Set NN & kinked tanh \bigstrut[t]\\
		SDA with $g=tanh$ & linear & 24 & Set NN & $\langle F(x), r(s) \rangle$ & Set NN & tanh \\
		SDA with no $g$ & linear & 24 & Set NN & $\langle F(x), r(s) \rangle$ & Set NN & - \\
		SDA with $w$ vector variable & linear & 24 & vector variable & $\langle F(x), r(s) \rangle$ & Set NN & kinked tanh \\
		SDE & linear & 24 & vector variable & Set NN & - & kinked tanh \\
		SDW & linear & 24 & Set NN & F(x) & - & - \\
		MNL with $w$ setnn & linear & 1 & Set NN & F(x) & - & - \\
		MNL \cite{mcfadden1973conditional} & linear & 1 & Scalar vector & F(x) & - & - \bigstrut[b]\\
		\hline
	\end{tabular}%
	}
	\label{tbl:ablation_details}%
	
\end{table}%

\begin{sidewaystable}
	\scalebox{0.8}{
		\begin{tabular}{|l|cccc|cccc|cccc|ccc|}
			\cline{2-16}\multicolumn{1}{r|}{} & \multicolumn{15}{c|}{Amadeus} \bigstrut\\
			\cline{2-16}\multicolumn{1}{r|}{} & \multicolumn{5}{c|}{Top-1} & \multicolumn{5}{c|}{Top-5} & \multicolumn{5}{c|}{MRR} \bigstrut\\
			\cline{2-16}\multicolumn{1}{r|}{} & \multicolumn{1}{c}{10} & \multicolumn{1}{c}{20} & \multicolumn{1}{c}{30} & \multicolumn{1}{c}{40} & \multicolumn{1}{c|}{50} & \multicolumn{1}{c}{10} & \multicolumn{1}{c}{20} & \multicolumn{1}{c}{30} & \multicolumn{1}{c}{40} & \multicolumn{1}{c|}{50} & \multicolumn{1}{c}{10} & \multicolumn{1}{c}{20} & \multicolumn{1}{c}{30} & \multicolumn{1}{c}{40} & \multicolumn{1}{c|}{50} \bigstrut\\
			\hline
			SDA & \multicolumn{1}{c}{\textbf{45.42}\tbltiny{ $\pm 0.5$}} & \multicolumn{1}{c}{\textbf{33.48}\tbltiny{ $\pm 0.3$}} & \multicolumn{1}{c}{\textbf{29.26}\tbltiny{ $\pm 0.0$}} & \multicolumn{1}{c}{26.57\tbltiny{ $\pm 0.3$}} & \multicolumn{1}{c|}{\textbf{23.23}\tbltiny{ $\pm 0.2$}} & \multicolumn{1}{c}{\textbf{93.37}\tbltiny{ $\pm 0.0$}} & \multicolumn{1}{c}{\textbf{80.40}\tbltiny{ $\pm 0.3$}} & \multicolumn{1}{c}{\textbf{73.77}\tbltiny{ $\pm 0.4$}} & \multicolumn{1}{c}{69.64\tbltiny{ $\pm 0.0$}} & \multicolumn{1}{c|}{\textbf{62.35}\tbltiny{ $\pm 0.1$}} & \multicolumn{1}{c}{\textbf{2.31}\tbltiny{ $\pm 0.3$}} & \multicolumn{1}{c}{\textbf{3.50}\tbltiny{ $\pm 0.0$}} & \multicolumn{1}{c}{\textbf{4.33}\tbltiny{ $\pm 0.2$}} & \multicolumn{1}{c}{4.93\tbltiny{ $\pm 0.4$}} & \multicolumn{1}{c|}{\textbf{6.37}\tbltiny{ $\pm 0.0$}} \bigstrut[t]\\
			SDA with $g=tanh$ & \multicolumn{1}{c}{39.10\tbltiny{ $\pm 0.4$}} & \multicolumn{1}{c}{31.54\tbltiny{ $\pm 0.3$}} & \multicolumn{1}{c}{27.02\tbltiny{ $\pm 0.4$}} & \multicolumn{1}{c}{25.98\tbltiny{ $\pm 0.2$}} & \multicolumn{1}{c|}{20.67\tbltiny{ $\pm 0.3$}} & \multicolumn{1}{c}{91.40\tbltiny{ $\pm 0.3$}} & \multicolumn{1}{c}{79.12\tbltiny{ $\pm 0.2$}} & \multicolumn{1}{c}{71.05\tbltiny{ $\pm 0.3$}} & \multicolumn{1}{c}{68.85\tbltiny{ $\pm 0.2$}} & \multicolumn{1}{c|}{59.50\tbltiny{ $\pm 0.2$}} & \multicolumn{1}{c}{2.55\tbltiny{ $\pm 0.0$}} & \multicolumn{1}{c}{3.64\tbltiny{ $\pm 0.0$}} & \multicolumn{1}{c}{4.60\tbltiny{ $\pm 0.0$}} & \multicolumn{1}{c}{5.04\tbltiny{ $\pm 0.0$}} & \multicolumn{1}{c|}{6.85\tbltiny{ $\pm 0.0$}} \\
			SDA with no $g$ & \multicolumn{1}{c}{41.38\tbltiny{ $\pm 0.5$}} & \multicolumn{1}{c}{33.26\tbltiny{ $\pm 0.3$}} & \multicolumn{1}{c}{28.89\tbltiny{ $\pm 0.3$}} & \multicolumn{1}{c}{27.28\tbltiny{ $\pm 0.2$}} & \multicolumn{1}{c|}{22.16\tbltiny{ $\pm 0.2$}} & \multicolumn{1}{c}{92.19\tbltiny{ $\pm 0.3$}} & \multicolumn{1}{c}{80.22\tbltiny{ $\pm 0.3$}} & \multicolumn{1}{c}{73.08\tbltiny{ $\pm 0.4$}} & \multicolumn{1}{c}{\textbf{69.92}\tbltiny{ $\pm 0.3$}} & \multicolumn{1}{c|}{61.19\tbltiny{ $\pm 0.4$}} & \multicolumn{1}{c}{2.45\tbltiny{ $\pm 0.0$}} & \multicolumn{1}{c}{3.53\tbltiny{ $\pm 0.0$}} & \multicolumn{1}{c}{4.39\tbltiny{ $\pm 0.0$}} & \multicolumn{1}{c}{\textbf{4.87}\tbltiny{ $\pm 0.0$}} & \multicolumn{1}{c|}{6.56\tbltiny{ $\pm 0.0$}} \\
			SDA with $w$ vector variable & \multicolumn{1}{c}{37.50\tbltiny{ $\pm 0.3$}} & \multicolumn{1}{c}{30.64\tbltiny{ $\pm 0.3$}} & \multicolumn{1}{c}{26.16\tbltiny{ $\pm 0.3$}} & \multicolumn{1}{c}{24.76\tbltiny{ $\pm 0.2$}} & \multicolumn{1}{c|}{15.81\tbltiny{ $\pm 0.2$}} & \multicolumn{1}{c}{87.31\tbltiny{ $\pm 0.3$}} & \multicolumn{1}{c}{76.53\tbltiny{ $\pm 0.2$}} & \multicolumn{1}{c}{69.76\tbltiny{ $\pm 0.4$}} & \multicolumn{1}{c}{66.67\tbltiny{ $\pm 0.2$}} & \multicolumn{1}{c|}{46.81\tbltiny{ $\pm 0.3$}} & \multicolumn{1}{c}{2.82\tbltiny{ $\pm 0.0$}} & \multicolumn{1}{c}{3.99\tbltiny{ $\pm 0.0$}} & \multicolumn{1}{c}{4.92\tbltiny{ $\pm 0.0$}} & \multicolumn{1}{c}{5.49\tbltiny{ $\pm 0.0$}} & \multicolumn{1}{c|}{12.23\tbltiny{ $\pm 0.0$}} \\
			SDE & \multicolumn{1}{c}{39.62\tbltiny{ $\pm 0.4$}} & \multicolumn{1}{c}{32.26\tbltiny{ $\pm 0.4$}} & \multicolumn{1}{c}{27.62\tbltiny{ $\pm 0.3$}} & \multicolumn{1}{c}{\textbf{26.73}\tbltiny{ $\pm 0.2$}} & \multicolumn{1}{c|}{20.62\tbltiny{ $\pm 0.2$}} & \multicolumn{1}{c}{91.70\tbltiny{ $\pm 0.3$}} & \multicolumn{1}{c}{79.75\tbltiny{ $\pm 0.3$}} & \multicolumn{1}{c}{72.18\tbltiny{ $\pm 0.5$}} & \multicolumn{1}{c}{69.82\tbltiny{ $\pm 0.2$}} & \multicolumn{1}{c|}{58.89\tbltiny{ $\pm 0.3$}} & \multicolumn{1}{c}{2.52\tbltiny{ $\pm 0.0$}} & \multicolumn{1}{c}{3.58\tbltiny{ $\pm 0.0$}} & \multicolumn{1}{c}{4.48\tbltiny{ $\pm 0.1$}} & \multicolumn{1}{c}{4.91\tbltiny{ $\pm 0.0$}} & \multicolumn{1}{c|}{6.90\tbltiny{ $\pm 0.1$}} \\
			SDW & \multicolumn{1}{c}{39.98\tbltiny{ $\pm 0.4$}} & \multicolumn{1}{c}{32.03\tbltiny{ $\pm 0.3$}} & \multicolumn{1}{c}{27.85\tbltiny{ $\pm 0.3$}} & \multicolumn{1}{c}{26.29\tbltiny{ $\pm 0.3$}} & \multicolumn{1}{c|}{20.15\tbltiny{ $\pm 0.2$}} & \multicolumn{1}{c}{91.89\tbltiny{ $\pm 0.3$}} & \multicolumn{1}{c}{79.57\tbltiny{ $\pm 0.2$}} & \multicolumn{1}{c}{71.82\tbltiny{ $\pm 0.3$}} & \multicolumn{1}{c}{68.62\tbltiny{ $\pm 0.3$}} & \multicolumn{1}{c|}{58.55\tbltiny{ $\pm 0.2$}} & \multicolumn{1}{c}{2.50\tbltiny{ $\pm 0.0$}} & \multicolumn{1}{c}{3.59\tbltiny{ $\pm 0.0$}} & \multicolumn{1}{c}{4.51\tbltiny{ $\pm 0.0$}} & \multicolumn{1}{c}{5.04\tbltiny{ $\pm 0.0$}} & \multicolumn{1}{c|}{6.97\tbltiny{ $\pm 0.0$}} \\
			MNL with $w$ setnn & \multicolumn{1}{c}{38.21\tbltiny{ $\pm 0.4$}} & \multicolumn{1}{c}{27.60\tbltiny{ $\pm 0.3$}} & \multicolumn{1}{c}{23.66\tbltiny{ $\pm 0.4$}} & \multicolumn{1}{c}{22.20\tbltiny{ $\pm 0.1$}} & \multicolumn{1}{c|}{18.41\tbltiny{ $\pm 0.3$}} & \multicolumn{1}{c}{90.90\tbltiny{ $\pm 0.4$}} & \multicolumn{1}{c}{75.49\tbltiny{ $\pm 0.3$}} & \multicolumn{1}{c}{66.46\tbltiny{ $\pm 0.4$}} & \multicolumn{1}{c}{63.28\tbltiny{ $\pm 0.2$}} & \multicolumn{1}{c|}{54.11\tbltiny{ $\pm 0.4$}} & \multicolumn{1}{c}{2.59\tbltiny{ $\pm 0.0$}} & \multicolumn{1}{c}{4.00\tbltiny{ $\pm 0.0$}} & \multicolumn{1}{c}{5.20\tbltiny{ $\pm 0.1$}} & \multicolumn{1}{c}{5.87\tbltiny{ $\pm 0.0$}} & \multicolumn{1}{c|}{8.26\tbltiny{ $\pm 0.1$}} \\
			MNL \cite{mcfadden1973conditional} & \multicolumn{1}{c}{38.42\tbltiny{ $\pm 0.5$}} & \multicolumn{1}{c}{27.93\tbltiny{ $\pm 0.4$}} & \multicolumn{1}{c}{23.54\tbltiny{ $\pm 0.3$}} & \multicolumn{1}{c}{22.31\tbltiny{ $\pm 0.1$}} & \multicolumn{1}{c|}{18.39\tbltiny{ $\pm 0.2$}} & \multicolumn{1}{c}{91.02\tbltiny{ $\pm 0.3$}} & \multicolumn{1}{c}{76.51\tbltiny{ $\pm 0.3$}} & \multicolumn{1}{c}{68.36\tbltiny{ $\pm 0.4$}} & \multicolumn{1}{c}{65.10\tbltiny{ $\pm 0.4$}} & \multicolumn{1}{c|}{56.20\tbltiny{ $\pm 0.3$}} & \multicolumn{1}{c}{2.57\tbltiny{ $\pm 0.0$}} & \multicolumn{1}{c}{3.92\tbltiny{ $\pm 0.0$}} & \multicolumn{1}{c}{4.94\tbltiny{ $\pm 0.0$}} & \multicolumn{1}{c}{5.60\tbltiny{ $\pm 0.1$}} & \multicolumn{1}{c|}{7.55\tbltiny{ $\pm 0.1$}} \bigstrut[b]\\
			\hline
			\multicolumn{1}{r}{} &   &   &   & \multicolumn{1}{r}{} &   &   &   & \multicolumn{1}{r}{} &   &   &   & \multicolumn{1}{r}{} &   &   & \multicolumn{1}{r}{} \bigstrut\\
			\cline{2-16}\multicolumn{1}{r|}{} & \multicolumn{12}{c|}{Expedia}                 & \multicolumn{3}{c|}{Outbrain} \bigstrut\\
			\cline{2-16}\multicolumn{1}{r|}{} & \multicolumn{4}{c|}{Top-1} & \multicolumn{4}{c|}{Top-5} & \multicolumn{4}{c|}{MRR} & \multicolumn{1}{c}{Top-1} & \multicolumn{1}{c}{Top-5} & \multicolumn{1}{c|}{MRR} \bigstrut\\
			\cline{2-16}\multicolumn{1}{r|}{} & \multicolumn{1}{c}{10} & \multicolumn{1}{c}{20} & \multicolumn{1}{c}{30} & \multicolumn{1}{c|}{40} & \multicolumn{1}{c}{10} & \multicolumn{1}{c}{20} & \multicolumn{1}{c}{30} & \multicolumn{1}{c|}{40} & \multicolumn{1}{c}{10} & \multicolumn{1}{c}{20} & \multicolumn{1}{c}{30} & \multicolumn{1}{c|}{40} & \multicolumn{1}{c}{12} & \multicolumn{1}{c}{12} & \multicolumn{1}{c|}{12} \bigstrut\\
			\hline
			SDA & \textbf{31.49}\tbltiny{ $\pm 0.2$} & 26.81\tbltiny{ $\pm 0.1$} & \textbf{21.96}\tbltiny{ $\pm 0.0$} & \textbf{18.36}\tbltiny{ $\pm 0.2$} & \textbf{86.91}\tbltiny{ $\pm 0.2$} & 73.06\tbltiny{ $\pm 0.0$} & \textbf{61.68}\tbltiny{ $\pm 0.2$} & \textbf{53.56}\tbltiny{ $\pm 0.4$} & \textbf{2.99}\tbltiny{ $\pm 0.0$} & 4.18\tbltiny{ $\pm 0.2$} & \textbf{5.99}\tbltiny{ $\pm 0.2$} & \textbf{7.65}\tbltiny{ $\pm 0.0$} & 38.04\tbltiny{ $\pm 0.3$} & 94.54\tbltiny{ $\pm 0.1$} & \textbf{2.42}\tbltiny{ $\pm 0.0$} \bigstrut[t]\\
			SDA with $g=tanh$ & 31.19\tbltiny{ $\pm 0.1$} & 26.26\tbltiny{ $\pm 0.2$} & 21.81\tbltiny{ $\pm 0.2$} & 18.10\tbltiny{ $\pm 0.2$} & 86.18\tbltiny{ $\pm 0.2$} & 72.61\tbltiny{ $\pm 0.3$} & 60.90\tbltiny{ $\pm 0.3$} & 52.91\tbltiny{ $\pm 0.2$} & 3.02\tbltiny{ $\pm 0.0$} & 4.24\tbltiny{ $\pm 0.0$} & 6.08\tbltiny{ $\pm 0.0$} & 7.81\tbltiny{ $\pm 0.0$} & \textbf{38.05}\tbltiny{ $\pm 0.3$} & \textbf{94.56}\tbltiny{ $\pm 0.2$} & 2.42\tbltiny{ $\pm 0.0$} \\
			SDA with no $g$ & 30.90\tbltiny{ $\pm 0.2$} & 26.52\tbltiny{ $\pm 0.2$} & 21.88\tbltiny{ $\pm 0.2$} & 18.20\tbltiny{ $\pm 0.2$} & 86.31\tbltiny{ $\pm 0.2$} & 72.80\tbltiny{ $\pm 0.2$} & 60.79\tbltiny{ $\pm 0.4$} & 52.96\tbltiny{ $\pm 0.2$} & 3.02\tbltiny{ $\pm 0.0$} & 4.21\tbltiny{ $\pm 0.0$} & 6.08\tbltiny{ $\pm 0.0$} & 7.73\tbltiny{ $\pm 0.0$} & 37.79\tbltiny{ $\pm 0.3$} & 94.48\tbltiny{ $\pm 0.1$} & 2.43\tbltiny{ $\pm 0.0$} \\
			SDA with $w$ vector variable & 25.05\tbltiny{ $\pm 0.2$} & 22.43\tbltiny{ $\pm 0.2$} & 18.24\tbltiny{ $\pm 0.2$} & 17.63\tbltiny{ $\pm 0.1$} & 80.67\tbltiny{ $\pm 0.2$} & 68.61\tbltiny{ $\pm 0.2$} & 55.40\tbltiny{ $\pm 0.3$} & 51.19\tbltiny{ $\pm 0.3$} & 3.46\tbltiny{ $\pm 0.0$} & 4.75\tbltiny{ $\pm 0.0$} & 7.04\tbltiny{ $\pm 0.0$} & 8.11\tbltiny{ $\pm 0.0$} & 37.98\tbltiny{ $\pm 0.3$} & 94.54\tbltiny{ $\pm 0.2$} & 2.42\tbltiny{ $\pm 0.0$} \\
			SDE & 31.47\tbltiny{ $\pm 0.2$} & \textbf{26.86}\tbltiny{ $\pm 0.3$} & 21.85\tbltiny{ $\pm 0.2$} & 18.14\tbltiny{ $\pm 0.2$} & 86.52\tbltiny{ $\pm 0.2$} & 73.04\tbltiny{ $\pm 0.4$} & 61.06\tbltiny{ $\pm 0.3$} & 52.75\tbltiny{ $\pm 0.3$} & 3.00\tbltiny{ $\pm 0.0$} & 4.20\tbltiny{ $\pm 0.0$} & 6.07\tbltiny{ $\pm 0.0$} & 7.80\tbltiny{ $\pm 0.0$} & 37.59\tbltiny{ $\pm 0.3$} & 94.26\tbltiny{ $\pm 0.1$} & 2.44\tbltiny{ $\pm 0.0$} \\
			SDW & 31.27\tbltiny{ $\pm 0.2$} & 26.80\tbltiny{ $\pm 0.2$} & 21.88\tbltiny{ $\pm 0.2$} & 18.22\tbltiny{ $\pm 0.2$} & 86.67\tbltiny{ $\pm 0.2$} & \textbf{73.08}\tbltiny{ $\pm 0.2$} & 61.07\tbltiny{ $\pm 0.3$} & 53.41\tbltiny{ $\pm 0.2$} & 2.99\tbltiny{ $\pm 0.0$} & \textbf{4.17}\tbltiny{ $\pm 0.0$} & 6.05\tbltiny{ $\pm 0.0$} & 7.73\tbltiny{ $\pm 0.0$} & 37.86\tbltiny{ $\pm 0.3$} & 94.48\tbltiny{ $\pm 0.2$} & 2.43\tbltiny{ $\pm 0.0$} \\
			MNL with $w$ setnn & 30.04\tbltiny{ $\pm 0.2$} & 25.28\tbltiny{ $\pm 0.3$} & 20.82\tbltiny{ $\pm 0.3$} & 16.64\tbltiny{ $\pm 0.4$} & 85.41\tbltiny{ $\pm 0.1$} & 71.88\tbltiny{ $\pm 0.3$} & 58.98\tbltiny{ $\pm 0.5$} & 49.54\tbltiny{ $\pm 0.5$} & 3.07\tbltiny{ $\pm 0.0$} & 4.34\tbltiny{ $\pm 0.0$} & 6.32\tbltiny{ $\pm 0.1$} & 8.48\tbltiny{ $\pm 0.1$} & 37.80\tbltiny{ $\pm 0.3$} & 94.52\tbltiny{ $\pm 0.1$} & 2.43\tbltiny{ $\pm 0.0$} \\
			MNL \cite{mcfadden1973conditional} & 30.06\tbltiny{ $\pm 0.2$} & 25.29\tbltiny{ $\pm 0.3$} & 20.61\tbltiny{ $\pm 0.4$} & 16.65\tbltiny{ $\pm 0.4$} & 86.34\tbltiny{ $\pm 0.1$} & 72.96\tbltiny{ $\pm 0.3$} & 60.94\tbltiny{ $\pm 0.5$} & 53.26\tbltiny{ $\pm 0.5$} & 3.07\tbltiny{ $\pm 0.0$} & 4.33\tbltiny{ $\pm 0.0$} & 6.35\tbltiny{ $\pm 0.1$} & 8.48\tbltiny{ $\pm 0.1$} & 37.74\tbltiny{ $\pm 0.3$} & 94.52\tbltiny{ $\pm 0.2$} & 2.43\tbltiny{ $\pm 0.0$} \bigstrut[b]\\
			\hline
		\end{tabular}%
	}
	\label{tbl:abl}
	\caption{Ablation Experiment Result.}
\end{sidewaystable}

\subsection{Full Result} \label{app:fulltbl}

We ran our experiments for the following number of maximum items:
\begin{enumerate}
	\item Amadeus: 10, 20, 30, 40, 50
	\item Expedia: 10, 20, 30, 40
	\item Outbrain: 12
\end{enumerate}

The result is shown in Table 6.

\begin{sidewaystable}
	\scalebox{0.8}{
		\begin{tabular}{|l|cccc|cccc|cccc|ccc|}
			\cline{2-16}\multicolumn{1}{c|}{} & \multicolumn{5}{c|}{Top-1} & \multicolumn{5}{c|}{Top-5} & \multicolumn{5}{c|}{MRR} \bigstrut\\
			\cline{2-16}\multicolumn{1}{c|}{} & 10 & 20 & 30 & \multicolumn{1}{c}{40} & \multicolumn{1}{c|}{50} & 10 & 20 & \multicolumn{1}{c}{30} & 40 & \multicolumn{1}{c|}{50} & 10 & \multicolumn{1}{c}{20} & 30 & 40 & 50 \bigstrut\\
			\hline
			SDA & \textbf{45.42}\tbltiny{ $\pm 0.5$} & \textbf{33.48}\tbltiny{ $\pm 0.3$} & \textbf{29.26}\tbltiny{ $\pm 0.0$} & \multicolumn{1}{c}{26.57\tbltiny{ $\pm 0.3$}} & \multicolumn{1}{c|}{\textbf{23.23}\tbltiny{ $\pm 0.2$}} & \textbf{93.37}\tbltiny{ $\pm 0.0$} & \textbf{80.40}\tbltiny{ $\pm 0.3$} & \multicolumn{1}{c}{\textbf{73.77}\tbltiny{ $\pm 0.4$}} & 69.64\tbltiny{ $\pm 0.0$} & \multicolumn{1}{c|}{\textbf{62.35}\tbltiny{ $\pm 0.1$}} & \textbf{2.31}\tbltiny{ $\pm 0.3$} & \multicolumn{1}{c}{\textbf{3.50}\tbltiny{ $\pm 0.0$}} & \textbf{4.33}\tbltiny{ $\pm 0.2$} & 4.93\tbltiny{ $\pm 0.4$} & \textbf{6.37}\tbltiny{ $\pm 0.0$} \bigstrut[t]\\
			SDE & 39.62\tbltiny{ $\pm 0.4$} & 32.26\tbltiny{ $\pm 0.4$} & 27.62\tbltiny{ $\pm 0.3$} & \multicolumn{1}{c}{\textbf{26.73}\tbltiny{ $\pm 0.2$}} & \multicolumn{1}{c|}{20.62\tbltiny{ $\pm 0.2$}} & 91.70\tbltiny{ $\pm 0.3$} & 79.75\tbltiny{ $\pm 0.3$} & \multicolumn{1}{c}{72.18\tbltiny{ $\pm 0.5$}} & \textbf{69.82}\tbltiny{ $\pm 0.2$} & \multicolumn{1}{c|}{58.89\tbltiny{ $\pm 0.3$}} & 2.52\tbltiny{ $\pm 0.0$} & \multicolumn{1}{c}{3.58\tbltiny{ $\pm 0.0$}} & 4.48\tbltiny{ $\pm 0.1$} & \textbf{4.91}\tbltiny{ $\pm 0.0$} & 6.90\tbltiny{ $\pm 0.1$} \\
			SDW & 39.98\tbltiny{ $\pm 0.4$} & 32.03\tbltiny{ $\pm 0.3$} & 27.85\tbltiny{ $\pm 0.3$} & \multicolumn{1}{c}{26.29\tbltiny{ $\pm 0.3$}} & \multicolumn{1}{c|}{20.15\tbltiny{ $\pm 0.2$}} & 91.89\tbltiny{ $\pm 0.3$} & 79.57\tbltiny{ $\pm 0.2$} & \multicolumn{1}{c}{71.82\tbltiny{ $\pm 0.3$}} & 68.62\tbltiny{ $\pm 0.3$} & \multicolumn{1}{c|}{58.55\tbltiny{ $\pm 0.2$}} & 2.50\tbltiny{ $\pm 0.0$} & \multicolumn{1}{c}{3.59\tbltiny{ $\pm 0.0$}} & 4.51\tbltiny{ $\pm 0.0$} & 5.04\tbltiny{ $\pm 0.0$} & 6.97\tbltiny{ $\pm 0.0$} \bigstrut[b]\\
			\hline
			MNL \cite{mcfadden1973conditional} & 38.42\tbltiny{ $\pm 0.5$} & 27.93\tbltiny{ $\pm 0.4$} & 23.54\tbltiny{ $\pm 0.3$} & \multicolumn{1}{c}{22.31\tbltiny{ $\pm 0.1$}} & \multicolumn{1}{c|}{18.39\tbltiny{ $\pm 0.2$}} & 91.02\tbltiny{ $\pm 0.3$} & 76.51\tbltiny{ $\pm 0.3$} & \multicolumn{1}{c}{68.36\tbltiny{ $\pm 0.4$}} & 65.10\tbltiny{ $\pm 0.4$} & \multicolumn{1}{c|}{56.20\tbltiny{ $\pm 0.3$}} & 2.57\tbltiny{ $\pm 0.0$} & \multicolumn{1}{c}{3.92\tbltiny{ $\pm 0.0$}} & 4.94\tbltiny{ $\pm 0.0$} & 5.60\tbltiny{ $\pm 0.1$} & 7.55\tbltiny{ $\pm 0.1$} \bigstrut[t]\\
			SVMRank \cite{joachims2006training} & 40.27\tbltiny{ $\pm 0.4$} & 28.17\tbltiny{ $\pm 0.3$} & 23.99\tbltiny{ $\pm 0.3$} & \multicolumn{1}{c}{23.02\tbltiny{ $\pm 0.2$}} & \multicolumn{1}{c|}{18.64\tbltiny{ $\pm 0.2$}} & 91.94\tbltiny{ $\pm 0.3$} & 76.82\tbltiny{ $\pm 0.2$} & \multicolumn{1}{c}{68.56\tbltiny{ $\pm 0.4$}} & 66.52\tbltiny{ $\pm 0.2$} & \multicolumn{1}{c|}{57.50\tbltiny{ $\pm 0.2$}} & 2.49\tbltiny{ $\pm 0.0$} & \multicolumn{1}{c}{3.87\tbltiny{ $\pm 0.0$}} & 4.85\tbltiny{ $\pm 0.0$} & 5.35\tbltiny{ $\pm 0.0$} & 7.16\tbltiny{ $\pm 0.0$} \\
			RankNet \cite{burges2005learning} & 37.44\tbltiny{ $\pm 0.7$} & 26.77\tbltiny{ $\pm 0.6$} & 23.81\tbltiny{ $\pm 0.3$} & \multicolumn{1}{c}{20.29\tbltiny{ $\pm 0.7$}} & \multicolumn{1}{c|}{16.99\tbltiny{ $\pm 0.5$}} & 84.67\tbltiny{ $\pm 1.7$} & 66.06\tbltiny{ $\pm 1.9$} & \multicolumn{1}{c}{61.59\tbltiny{ $\pm 0.9$}} & 49.45\tbltiny{ $\pm 1.8$} & \multicolumn{1}{c|}{44.98\tbltiny{ $\pm 2.6$}} & 3.02\tbltiny{ $\pm 0.1$} & \multicolumn{1}{c}{4.98\tbltiny{ $\pm 0.2$}} & 5.96\tbltiny{ $\pm 0.1$} & 8.35\tbltiny{ $\pm 0.3$} & 11.07\tbltiny{ $\pm 0.7$} \bigstrut[b]\\
			\hline
			Mixed MNL \cite{train2009discrete} & 37.96\tbltiny{ $\pm 0.3$} & 27.00\tbltiny{ $\pm 0.2$} & 22.98\tbltiny{ $\pm 0.3$} & \multicolumn{1}{c}{21.68\tbltiny{ $\pm 0.2$}} & \multicolumn{1}{c|}{17.67\tbltiny{ $\pm 0.3$}} & 90.40\tbltiny{ $\pm 0.3$} & 74.80\tbltiny{ $\pm 0.2$} & \multicolumn{1}{c}{65.87\tbltiny{ $\pm 0.4$}} & 62.50\tbltiny{ $\pm 0.3$} & \multicolumn{1}{c|}{52.87\tbltiny{ $\pm 0.4$}} & 2.62\tbltiny{ $\pm 0.0$} & \multicolumn{1}{c}{4.09\tbltiny{ $\pm 0.0$}} & 5.26\tbltiny{ $\pm 0.0$} & 6.00\tbltiny{ $\pm 0.1$} & 8.39\tbltiny{ $\pm 0.1$} \bigstrut[t]\\
			AdaRank \cite{adarank} & 37.27\tbltiny{ $\pm 0.4$} & 25.79\tbltiny{ $\pm 0.3$} & 18.79\tbltiny{ $\pm 0.7$} & \multicolumn{1}{c}{15.79\tbltiny{ $\pm 0.5$}} & \multicolumn{1}{c|}{11.89\tbltiny{ $\pm 0.2$}} & 72.34\tbltiny{ $\pm 0.3$} & 58.28\tbltiny{ $\pm 0.3$} & \multicolumn{1}{c}{51.64\tbltiny{ $\pm 0.6$}} & 47.85\tbltiny{ $\pm 1.4$} & \multicolumn{1}{c|}{39.08\tbltiny{ $\pm 0.2$}} & 4.03\tbltiny{ $\pm 0.0$} & \multicolumn{1}{c}{5.75\tbltiny{ $\pm 0.0$}} & 7.14\tbltiny{ $\pm 0.1$} & 8.05\tbltiny{ $\pm 0.3$} & 11.55\tbltiny{ $\pm 0.1$} \\
			Deep Sets \cite{zaheer2017deep} & 40.36\tbltiny{ $\pm 0.5$} & 31.02\tbltiny{ $\pm 0.5$} & 26.66\tbltiny{ $\pm 0.4$} & \multicolumn{1}{c}{25.48\tbltiny{ $\pm 0.3$}} & \multicolumn{1}{c|}{20.55\tbltiny{ $\pm 0.3$}} & 91.92\tbltiny{ $\pm 0.3$} & 79.31\tbltiny{ $\pm 0.2$} & \multicolumn{1}{c}{71.18\tbltiny{ $\pm 0.4$}} & 68.76\tbltiny{ $\pm 0.4$} & \multicolumn{1}{c|}{59.88\tbltiny{ $\pm 0.4$}} & 2.48\tbltiny{ $\pm 0.0$} & \multicolumn{1}{c}{3.64\tbltiny{ $\pm 0.0$}} & 4.58\tbltiny{ $\pm 0.0$} & 5.01\tbltiny{ $\pm 0.0$} & 6.75\tbltiny{ $\pm 0.1$} \bigstrut[b]\\
			\hline
			Price/Quality & 36.44\tbltiny{ $\pm 0.3$} & 25.44\tbltiny{ $\pm 0.2$} & 22.40\tbltiny{ $\pm 0.2$} & \multicolumn{1}{c}{20.26\tbltiny{ $\pm 0.2$}} & \multicolumn{1}{c|}{16.11\tbltiny{ $\pm 0.1$}} & 87.23\tbltiny{ $\pm 0.2$} & 67.77\tbltiny{ $\pm 0.2$} & \multicolumn{1}{c}{58.86\tbltiny{ $\pm 0.2$}} & 54.44\tbltiny{ $\pm 0.2$} & \multicolumn{1}{c|}{45.43\tbltiny{ $\pm 0.3$}} & 2.79\tbltiny{ $\pm 0.0$} & \multicolumn{1}{c}{4.86\tbltiny{ $\pm 0.0$}} & 6.32\tbltiny{ $\pm 0.0$} & 7.27\tbltiny{ $\pm 0.0$} & 10.90\tbltiny{ $\pm 0.1$} \bigstrut[t]\\
			Random & 25.15\tbltiny{ $\pm 0.5$} & 14.78\tbltiny{ $\pm 0.2$} & 11.58\tbltiny{ $\pm 0.2$} & \multicolumn{1}{c}{9.91\tbltiny{ $\pm 0.2$}} & \multicolumn{1}{c|}{6.24\tbltiny{ $\pm 0.2$}} & 32.87\tbltiny{ $\pm 0.6$} & 42.60\tbltiny{ $\pm 0.3$} & \multicolumn{1}{c}{34.54\tbltiny{ $\pm 0.2$}} & 14.20\tbltiny{ $\pm 0.2$} & \multicolumn{1}{c|}{11.04\tbltiny{ $\pm 0.2$}} & 6.49\tbltiny{ $\pm 0.0$} & \multicolumn{1}{c}{8.98\tbltiny{ $\pm 0.1$}} & 12.03\tbltiny{ $\pm 0.1$} & 18.11\tbltiny{ $\pm 0.1$} & 23.55\tbltiny{ $\pm 0.1$} \bigstrut[b]\\
			\hline
			\multicolumn{1}{r}{} &   &   &   & \multicolumn{1}{r}{} &   &   &   & \multicolumn{1}{r}{} &   &   &   & \multicolumn{1}{r}{} &   &   & \multicolumn{1}{r}{} \bigstrut\\
			\cline{2-16}\multicolumn{1}{r|}{} & \multicolumn{12}{c|}{Expedia}                 & \multicolumn{3}{c|}{Outbrain} \bigstrut\\
			\cline{2-16}\multicolumn{1}{r|}{} & \multicolumn{4}{c|}{Top-1} & \multicolumn{4}{c|}{Top-5} & \multicolumn{4}{c|}{MRR} & Top-1 & Top-5 & MRR \bigstrut\\
			\cline{2-16}\multicolumn{1}{r|}{} & 10 & 20 & 30 & 40 & 10 & 20 & 30 & 40 & 10 & 20 & 30 & 40 & 12 & 12 & 12 \bigstrut\\
			\hline
			SDA & \textbf{31.49}\tbltiny{ $\pm 0.2$} & 26.81\tbltiny{ $\pm 0.1$} & \textbf{21.96}\tbltiny{ $\pm 0.0$} & \textbf{18.36}\tbltiny{ $\pm 0.2$} & \textbf{86.91}\tbltiny{ $\pm 0.2$} & 73.06\tbltiny{ $\pm 0.0$} & \textbf{61.68}\tbltiny{ $\pm 0.2$} & \textbf{53.56}\tbltiny{ $\pm 0.4$} & \textbf{2.99}\tbltiny{ $\pm 0.0$} & 4.18\tbltiny{ $\pm 0.2$} & \textbf{5.99}\tbltiny{ $\pm 0.2$} & \textbf{7.65}\tbltiny{ $\pm 0.0$} & \textbf{38.04}\tbltiny{ $\pm 0.3$} & \textbf{94.54}\tbltiny{ $\pm 0.1$} & \textbf{2.42}\tbltiny{ $\pm 0.0$} \bigstrut[t]\\
			SDE & 31.47\tbltiny{ $\pm 0.2$} & \textbf{26.86}\tbltiny{ $\pm 0.3$} & 21.85\tbltiny{ $\pm 0.2$} & 18.14\tbltiny{ $\pm 0.2$} & 86.52\tbltiny{ $\pm 0.2$} & 73.04\tbltiny{ $\pm 0.4$} & 61.06\tbltiny{ $\pm 0.3$} & 52.75\tbltiny{ $\pm 0.3$} & 3.00\tbltiny{ $\pm 0.0$} & 4.20\tbltiny{ $\pm 0.0$} & 6.07\tbltiny{ $\pm 0.0$} & 7.80\tbltiny{ $\pm 0.0$} & 37.59\tbltiny{ $\pm 0.3$} & 94.26\tbltiny{ $\pm 0.1$} & 2.44\tbltiny{ $\pm 0.0$} \\
			SDW & 31.27\tbltiny{ $\pm 0.2$} & 26.80\tbltiny{ $\pm 0.2$} & 21.88\tbltiny{ $\pm 0.2$} & 18.22\tbltiny{ $\pm 0.2$} & 86.67\tbltiny{ $\pm 0.2$} & 73.08\tbltiny{ $\pm 0.2$} & 61.07\tbltiny{ $\pm 0.3$} & 53.41\tbltiny{ $\pm 0.2$} & 2.99\tbltiny{ $\pm 0.0$} & \textbf{4.17}\tbltiny{ $\pm 0.0$} & 6.05\tbltiny{ $\pm 0.0$} & 7.73\tbltiny{ $\pm 0.0$} & 37.86\tbltiny{ $\pm 0.3$} & 94.48\tbltiny{ $\pm 0.2$} & 2.43\tbltiny{ $\pm 0.0$} \bigstrut[b]\\
			\hline
			MNL \cite{mcfadden1973conditional} & 30.06\tbltiny{ $\pm 0.2$} & 25.29\tbltiny{ $\pm 0.3$} & 20.61\tbltiny{ $\pm 0.4$} & 16.65\tbltiny{ $\pm 0.4$} & 86.34\tbltiny{ $\pm 0.1$} & 72.96\tbltiny{ $\pm 0.3$} & 60.94\tbltiny{ $\pm 0.5$} & 53.26\tbltiny{ $\pm 0.5$} & 3.07\tbltiny{ $\pm 0.0$} & 4.33\tbltiny{ $\pm 0.0$} & 6.35\tbltiny{ $\pm 0.1$} & 8.48\tbltiny{ $\pm 0.1$} & 37.74\tbltiny{ $\pm 0.3$} & 94.52\tbltiny{ $\pm 0.2$} & 2.43\tbltiny{ $\pm 0.0$} \bigstrut[t]\\
			SVMRank \cite{joachims2006training} & 31.28\tbltiny{ $\pm 0.2$} & 26.64\tbltiny{ $\pm 0.3$} & 21.93\tbltiny{ $\pm 0.2$} & 18.03\tbltiny{ $\pm 0.2$} & 86.24\tbltiny{ $\pm 0.1$} & \textbf{73.17}\tbltiny{ $\pm 0.3$} & 60.53\tbltiny{ $\pm 0.2$} & 52.25\tbltiny{ $\pm 0.2$} & 3.01\tbltiny{ $\pm 0.0$} & 4.19\tbltiny{ $\pm 0.0$} & 6.12\tbltiny{ $\pm 0.0$} & 7.95\tbltiny{ $\pm 0.0$} & 37.68\tbltiny{ $\pm 0.3$} & 94.46\tbltiny{ $\pm 0.1$} & 2.43\tbltiny{ $\pm 0.0$} \\
			RankNet \cite{burges2005learning} & 23.82\tbltiny{ $\pm 0.5$} & 18.60\tbltiny{ $\pm 0.7$} & 11.48\tbltiny{ $\pm 0.6$} & 11.54\tbltiny{ $\pm 0.4$} & 81.85\tbltiny{ $\pm 0.6$} & 62.91\tbltiny{ $\pm 0.9$} & 43.09\tbltiny{ $\pm 1.1$} & 38.49\tbltiny{ $\pm 0.9$} & 3.43\tbltiny{ $\pm 0.0$} & 5.21\tbltiny{ $\pm 0.1$} & 8.72\tbltiny{ $\pm 0.2$} & 10.49\tbltiny{ $\pm 0.2$} & 35.32\tbltiny{ $\pm 0.8$} & 91.55\tbltiny{ $\pm 0.7$} & 2.65\tbltiny{ $\pm 0.1$} \bigstrut[b]\\
			\hline
			Mixed MNL \cite{train2009discrete} & 27.28\tbltiny{ $\pm 0.6$} & 22.00\tbltiny{ $\pm 0.7$} & 18.32\tbltiny{ $\pm 0.4$} & 13.91\tbltiny{ $\pm 0.6$} & 84.24\tbltiny{ $\pm 0.3$} & 68.31\tbltiny{ $\pm 0.7$} & 55.41\tbltiny{ $\pm 0.6$} & 43.55\tbltiny{ $\pm 1.0$} & 3.22\tbltiny{ $\pm 0.0$} & 4.67\tbltiny{ $\pm 0.1$} & 6.81\tbltiny{ $\pm 0.1$} & 9.45\tbltiny{ $\pm 0.2$} & 37.72\tbltiny{ $\pm 0.3$} & 94.42\tbltiny{ $\pm 0.1$} & 2.43\tbltiny{ $\pm 0.0$} \bigstrut[t]\\
			AdaRank \cite{adarank} & 26.70\tbltiny{ $\pm 0.2$} & 22.57\tbltiny{ $\pm 0.3$} & 17.47\tbltiny{ $\pm 0.2$} & 14.14\tbltiny{ $\pm 0.2$} & 83.21\tbltiny{ $\pm 0.2$} & 68.14\tbltiny{ $\pm 0.3$} & 54.46\tbltiny{ $\pm 0.3$} & 44.89\tbltiny{ $\pm 0.2$} & 3.29\tbltiny{ $\pm 0.0$} & 4.71\tbltiny{ $\pm 0.0$} & 7.01\tbltiny{ $\pm 0.0$} & 9.33\tbltiny{ $\pm 0.0$} & 37.47\tbltiny{ $\pm 0.3$} & 94.40\tbltiny{ $\pm 0.2$} & 2.44\tbltiny{ $\pm 0.0$} \\
			Deep Sets \cite{zaheer2017deep} & 29.87\tbltiny{ $\pm 0.3$} & 25.64\tbltiny{ $\pm 0.2$} & 20.95\tbltiny{ $\pm 0.3$} & 16.74\tbltiny{ $\pm 0.3$} & 86.26\tbltiny{ $\pm 0.2$} & 72.55\tbltiny{ $\pm 0.3$} & 60.25\tbltiny{ $\pm 0.3$} & 51.35\tbltiny{ $\pm 0.3$} & 3.06\tbltiny{ $\pm 0.0$} & 4.26\tbltiny{ $\pm 0.0$} & 6.19\tbltiny{ $\pm 0.0$} & 8.08\tbltiny{ $\pm 0.1$} & 37.51\tbltiny{ $\pm 0.3$} & 94.30\tbltiny{ $\pm 0.1$} & 2.44\tbltiny{ $\pm 0.0$} \bigstrut[b]\\
			\cline{1-1}Price/Quality & 17.92\tbltiny{ $\pm 0.1$} & 13.24\tbltiny{ $\pm 0.2$} & 9.92\tbltiny{ $\pm 0.1$} & 7.80\tbltiny{ $\pm 0.1$} & 77.67\tbltiny{ $\pm 0.1$} & 56.00\tbltiny{ $\pm 0.2$} & 42.50\tbltiny{ $\pm 0.2$} & 32.94\tbltiny{ $\pm 0.3$} & 3.79\tbltiny{ $\pm 0.0$} & 5.84\tbltiny{ $\pm 0.0$} & 8.60\tbltiny{ $\pm 0.0$} & 11.21\tbltiny{ $\pm 0.1$} & 24.17\tbltiny{ $\pm 0.1$} & 25.08\tbltiny{ $\pm 0.1$} & 8.13\tbltiny{ $\pm 0.0$} \bigstrut[t]\\
			Random & 14.13\tbltiny{ $\pm 0.1$} & 9.68\tbltiny{ $\pm 0.2$} & 6.89\tbltiny{ $\pm 0.1$} & 5.05\tbltiny{ $\pm 0.1$} & 32.35\tbltiny{ $\pm 0.2$} & 33.44\tbltiny{ $\pm 0.2$} & 24.16\tbltiny{ $\pm 0.2$} & 13.25\tbltiny{ $\pm 0.2$} & 6.38\tbltiny{ $\pm 0.0$} & 8.65\tbltiny{ $\pm 0.0$} & 12.01\tbltiny{ $\pm 0.0$} & 17.89\tbltiny{ $\pm 0.1$} & 22.21\tbltiny{ $\pm 0.1$} & 23.10\tbltiny{ $\pm 0.1$} & 8.32\tbltiny{ $\pm 0.0$} \bigstrut[b]\\
			\hline
		\end{tabular}%
	}
	\label{tbl:full}
	\caption{Full experimental results.}
\end{sidewaystable}

\end{appendices}

\end{document}